\documentclass{article}

\usepackage{iclr2027_conference,times}

\usepackage{amsmath,amssymb,amsthm}
\usepackage{graphicx}
\usepackage{booktabs}
\usepackage{multirow}
\usepackage{float}
\usepackage{placeins}
\usepackage{wrapfig}
\usepackage[dvipsnames,table]{xcolor}
\usepackage{threeparttable}
\usepackage{array}
\usepackage{enumitem}
\usepackage{tikz}
\usepackage{tcolorbox}
\tcbuselibrary{breakable}
\usetikzlibrary{arrows.meta,backgrounds,calc,fit,positioning,shapes.geometric}
\usepackage{pifont}
\usepackage{hyperref}
\usepackage{url}

\definecolor{ptblue}{HTML}{4477AA}
\definecolor{ptcyan}{HTML}{66CCEE}
\definecolor{ptgreen}{HTML}{228833}
\definecolor{ptpurple}{HTML}{AA3377}
\definecolor{ptgrey}{HTML}{BBBBBB}
\definecolor{headerblue}{RGB}{220,235,241}
\definecolor{groupgray}{RGB}{244,246,248}
\definecolor{ourscolor}{RGB}{246,238,215}
\definecolor{onlinecolor}{RGB}{237,224,186}
\definecolor{increasecolor}{HTML}{FF4B33}
\definecolor{decreasecolor}{HTML}{404999}
\definecolor{keyaccent}{HTML}{1F3A5F}
\newcommand{\pointmark}[1]{\textcolor{MidnightBlue}{\ding{#1}}}
\newcommand{\key}[1]{\textbf{#1}}
\newcommand{\finding}[1]{\textcolor{keyaccent}{\textbf{#1}}}

\newcommand{\best}[1]{\textbf{#1}}
\newcommand{\ours}{RIVET}  
\newcommand{\oursfull}{Response-Informed Value Estimation for Turn-Level Routing}

\newcommand{\Ret}{\mathsf{R}}
\newcommand{\E}{\mathbb{E}}
\newcommand{\Var}{\operatorname{Var}}
\newcommand{\Ustar}{U^{\star}}
\newcommand{\costref}{c_{\mathrm{ref}}}
\newcommand{\sigfut}{\sigma_{\mathrm{future}}^{2}}
\newcommand{\Qroute}{Q_{\mathrm{route}}}
\newcommand{\Qresp}{Q_{\mathrm{resp}}}
\newcommand{\froute}{f_{\mathrm{route}}}
\newcommand{\fresp}{f_{\mathrm{resp}}}
\newcommand{\Lroute}{\mathcal{L}_{\mathrm{route}}}
\newcommand{\Lresp}{\mathcal{L}_{\mathrm{resp}}}
\newtheorem{theorem}{Theorem}[section]
\newtheorem{proposition}[theorem]{Proposition}
\newtheorem{assumption}[theorem]{Assumption}
\theoremstyle{remark}
\newtheorem{remark}[theorem]{Remark}
\newcommand{\metricup}[2]{%
  #1\raisebox{-0.45ex}{\hspace{0.15em}\scriptsize
  \textcolor{increasecolor}{\ensuremath{\uparrow}\,#2}}}
\newcommand{\metricdown}[2]{%
  #1\raisebox{-0.45ex}{\hspace{0.15em}\scriptsize
  \textcolor{decreasecolor}{\ensuremath{\downarrow}\,#2}}}
\newcommand{\tablevsep}{\hspace{3pt}{\color{black!25}\vrule width 0.3pt}\hspace{3pt}}
\newcommand{\offlineicon}{%
  \tikz[baseline=-0.45ex]{\node[
    star,star points=5,star point ratio=2.15,
    minimum size=1.35ex,inner sep=0pt,
    draw=MidnightBlue,fill=white,line width=0.32pt] {};}
  \hspace{0.20em}}
\newcommand{\onlineicon}{%
  \tikz[baseline=-0.45ex]{\node[
    star,star points=5,star point ratio=2.15,
    minimum size=1.35ex,inner sep=0pt,
    draw=MidnightBlue,fill=Goldenrod,line width=0.32pt] {};}
  \hspace{0.20em}}

\hypersetup{
  colorlinks=true,
  linkcolor=MidnightBlue,
  citecolor=MidnightBlue,
  urlcolor=MidnightBlue
}

\title{Response-Informed Value Estimation\\for Cost-Aware Multi-Turn LLM Routing}
\author{%
Haotian Yan$^{1,4}$, Jiaqi Liu$^{2}$, Bingyi Jing$^{3,4}$, Tianjun Yao$^{4}$\\
\textsuperscript{1}Southwest University of Finance and Economics\\
\textsuperscript{2}Dalian University of Technology\\
\textsuperscript{3}The Chinese University of Hong Kong, Shenzhen\\
\textsuperscript{4}Shenzhen Loop Area Institute%
}

\iclrfinalcopy

\begin{document}

\hyphenpenalty=10000
\exhyphenpenalty=10000
\emergencystretch=1.5em

\maketitle

\begin{abstract}
Multi-turn LLM routing selects a model at every turn before its response is
available, yet task success is revealed only when the interaction ends,
leaving each intermediate choice without a label of its own. Logged
trajectories, however, retain the response each call produced, which exposes
what each turn required and how well the chosen model served it, evidence
that outcomes alone hide. We introduce \ours{}
(\oursfull{}), which exploits this post-call evidence during training while
keeping routing strictly pre-call at deployment. \ours{} jointly trains two
value estimators over a shared state--model representation: a deployable
Routing Value Estimator that scores candidate models from pre-call
information, and a training-only Response-Informed Value Estimator that
additionally observes the realized response. Conditioning on the response
turns response-induced return variation from residual noise into an observed
covariate, and we show that this reduces the leading asymptotic estimation
error of the shared representation. On the MINT benchmark, \ours{} solves
more held-out tasks than the strongest multi-turn baseline while using 74.7\%
less total call cost.
\end{abstract}

\section{Introduction}
\label{sec:introduction}

Routing among language models of different capability and price is a
standard way to trade accuracy for cost: a learned router predicts which
model is worth calling for a
query~\citep{ong2025routellm,ding2024hybrid,chen2024routerdc,
zhuang2025embedllm,feng2024graphrouter,zhang2025avengerspro}, cascades
escalate only when a cheaper answer looks
insufficient~\citep{chen2024frugalgpt,aggarwal2024automix}, and benchmarks
score the resulting accuracy--cost frontier~\citep{hu2024routerbench}. Yet
the workloads where routing matters most are increasingly multi-turn: agents
reason, invoke tools, observe feedback, and decide how to
proceed~\citep{wang2024mint}, so every turn is a fresh routing decision, and
what each turn demands of a model is not something a static strong/weak
ranking can tell (Figure~\ref{fig:response-motivation}).

Multi-turn routing is harder along three axes. \key{Credit assignment is
delayed}: only the final outcome is graded, so intermediate decisions receive
no label of their own. \key{Counterfactual evaluation is
trajectory-dependent}: changing the model at one turn changes every later
state, so alternative routes cannot be scored against a shared context.
\key{Trajectory collection is expensive}: each trajectory is a sequence of
paid model calls, so on-policy learning buys every sample at multi-turn
prices. A multi-turn router must therefore propagate delayed consequences
while extracting as much as possible from the experience already logged.

Logged experience, however, contains more than outcomes. The response
recorded at each call shows what that turn actually required and whether the
chosen model delivered it. Figure~\ref{fig:response-motivation} illustrates
this on a four-turn recommendation task. \key{A response exposes decision
quality, not just spending}: whether it honors the user's hard constraint,
picks up the preference just stated, keeps what earlier turns already
settled, and answers at the length the turn deserves. None of this is
visible in the final outcome, and all of it matters for which model to call
next. \key{No single model wins every turn}: the small model succeeds where
it keeps the constraint, the strongest model fails where it drifts from the
conversation, and the mid-sized model over-deliberates or over-explains.
These failure modes differ from turn to turn, so no fixed strong/weak
ranking captures them. Responses therefore carry turn-level value beyond
cost: they show what the turn asked for and whether the model answered it.
A pre-call router must anticipate these effects from state--model
information alone. This motivates using the realized response as privileged
auxiliary information during training, while keeping routing strictly
pre-call at deployment.

\begin{figure}[H]
  \centering
  \includegraphics[width=0.98\textwidth]{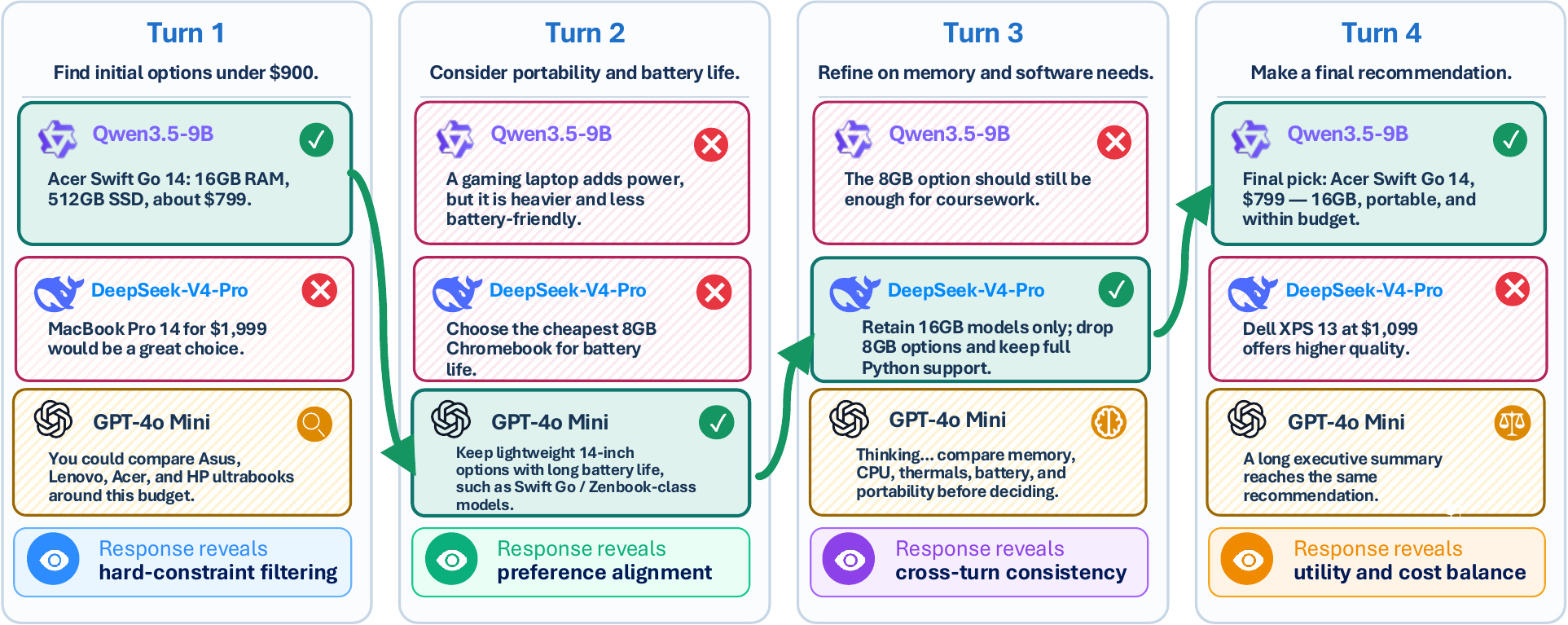}
  \caption{Intermediate responses reveal what outcomes hide. Three candidate
  models on a four-turn recommendation task. Each turn demands a different
  ability (bottom labels), and the response exposes whether the model met it:
  check and cross mark acceptable and unacceptable responses, and the
  remaining icons mark responses that are vague, over-deliberated, or
  needlessly verbose. The green path is the cost-efficient route and switches
  models across turns. Illustrative; each logged turn observes only the
  selected model's response.}
  \label{fig:response-motivation}
\end{figure}

To exploit both properties of logged experience, we introduce
\textbf{\ours{}} (\oursfull{}), a sample-efficient off-policy value-learning
framework for multi-turn LLM routing. \ours{} formulates turn-level model
selection as a cost-aware sequential decision problem with trajectory-level
supervision, reuses logged transitions to propagate future routing
consequences through Bellman updates, and requires only a small online
adaptation budget instead of large-scale on-policy rollout collection.

\ours{} uses the logged responses through two value estimators that share
one trainable state--model representation. The deployed \emph{Routing Value
Estimator} (RVE) conditions only on the current state and candidate model;
the training-only \emph{Response-Informed Value Estimator} (RIVE)
additionally observes the executed response. Both predict the same future
return, but RIVE faces a cleaner problem: response-dependent variation that
RVE must treat as noise is an observed input for RIVE. Because the
representation is shared, what RIVE learns on this cleaner problem improves
the representation RVE relies on; Section~\ref{sec:theory} shows that joint
training suppresses the response-resolved component of return uncertainty in
the estimation error of the shared representation. \key{At deployment RIVE is
discarded}: routing remains strictly pre-call, with no additional model
invocation or post-call stage.

On the MINT multi-turn benchmark, \ours{} improves task accuracy over MTRouter
while reducing model-call cost by 74.7\%. Its cost advantage transfers across
three out-of-distribution benchmarks while maintaining competitive task
performance. Controlled experiments show that the gain depends on RIVE having
access to the executed response (Section~\ref{sec:response-mechanism}).

Our contributions are threefold:
\begin{itemize}[leftmargin=1.6em,itemsep=1pt,topsep=2pt]
  \item[\pointmark{172}] \textbf{Multi-turn routing as off-policy value learning.} We formulate
  turn-level model selection in agentic interaction as a cost-aware sequential
  value-learning problem with delayed trajectory-level supervision, and show
  that a small logged corpus plus limited online adaptation suffices for
  effective routing.
  \item[\pointmark{173}] \textbf{Response-informed auxiliary value learning.} A training-only
  estimator that additionally conditions on the executed response shares a
  state--model representation with the deployable pre-call estimator,
  providing an auxiliary signal without changing the deployment routing path.
  We show this reduces the leading asymptotic estimation error of the shared
  representation.
  \item[\pointmark{174}] \textbf{Sample-efficient multi-turn routing results.} From a small
  logged corpus and a limited online budget, \ours{} reaches the high-accuracy
  end of the measured accuracy--cost frontier on MINT, retains strong cost
  efficiency under out-of-distribution evaluation, and maintains its
  accuracy--cost trade-off under candidate-model unavailability.
\end{itemize}

\section{Related Work}
\label{sec:related-work}

\paragraph{LLM routing.}
Most routers choose a model for one query and stop. Pre-call predictors
estimate query--model compatibility from preference
data~\citep{ong2025routellm}, predicted difficulty~\citep{ding2024hybrid},
contrastive query and model embeddings~\citep{chen2024routerdc,zhuang2025embedllm},
task--query--model graphs~\citep{feng2024graphrouter}, or query
clusters~\citep{zhang2025avengerspro}; cascades inspect a cheap answer
before escalating~\citep{chen2024frugalgpt,aggarwal2024automix}; bandits
adapt the choice online~\citep{dai2025costeffective}; and RouterBench
standardizes accuracy--cost evaluation~\citep{hu2024routerbench}. All score
candidates within a fixed query context. Agentic workloads make routing
sequential, yet the multi-turn setting remains sparsely
studied~\citep{moslem2026dynamicmodelroutingcascading}. Router-R1,
R2-Reasoner, and xRouter instantiate the router as an LLM that interleaves
reasoning with model or tool calls and train it with on-policy policy
gradients~\citep{zhang2025routerr1,shao2025routeandreason,qian2025xroutertrainingcostawarellms,shao2024grpo};
GraphPlanner assembles model--role workflows from prior
interactions~\citep{feng2026graphplanner}; SWE-Router watches a cheap agent's
partial trajectory before a one-shot escalation~\citep{son2026swerouter};
DialRouter distills MCTS rollouts into a dialogue router by behavior
cloning~\citep{zhang2026myopic}; and MTRouter, which shares our decision
unit, selects a model at every turn from a supervised outcome estimator over
history--model embeddings~\citep{zhang2026mtrouter}. These methods buy each
training sample with fresh rollouts, imitate search, or regress outcomes
without propagating value across turns; \key{none learns turn-level action values
off-policy from logged trajectories}.

\paragraph{Value-based RL and privileged supervision.}
Value-based methods are established for language decisions: ILQL learns
token-level Q-functions offline~\citep{snell2022offline}, ArCHer shows that a
turn-level TD critic is far more sample-efficient than on-policy training for
multi-turn agents~\citep{zhou2024archer}, and conservative Q-learning makes
bootstrapping from fixed logs safe~\citep{kumar2020conservative}. \ours{}
brings this machinery to the discrete model-selection action space with
Double DQN targets~\citep{vanhasselt2016deep} and CQL regularization. Its
second ingredient draws on learning with privileged
information~\citep{vapnik2009new,lopez2015unifying} and asymmetric
actor--critic methods, where a training-time critic observes what the
deployed policy cannot~\citep{pinto2017asymmetric,baisero2021unbiased}, cast
here as an auxiliary task~\citep{jaderberg2016reinforcement}. Cascades and
SWE-Router also exploit the response, but at deployment, paying for a call
before deciding; \key{\ours{} uses the response only during training}, so the
deployed router stays pre-call, and we characterize when this auxiliary
supervision improves the shared representation (Section~\ref{sec:theory}).

\section{Problem Formulation}
\label{sec:problem-formulation}

We formulate multi-turn LLM routing as a cost-aware Markov decision process.
For a task $q$ and candidate model set $\mathcal{M}$, the state at turn $t$
is $s_t=(q,\mathcal{H}_{<t})$, where $\mathcal{H}_{<t}$ is the interaction
history before the current model call. The router selects a model
$m_t\sim\pi(\cdot\mid s_t)$, which generates a response
$o_t\sim p_{m_t}(\cdot\mid s_t)$; the environment then produces feedback
$f_t$, and the interaction transitions to $s_{t+1}=T(s_t,o_t,f_t)$. Each
call incurs a cost $c_t=c(s_t,m_t,o_t)$. Let $d_t\in\{0,1\}$ indicate
episode termination and $y\in\{0,1\}$ the final task outcome. The per-turn
reward is
\begin{equation}
  r_t=-\lambda\frac{c_t}{\costref}+d_t y,
  \label{eq:immediate-reward}
\end{equation}
where $\costref>0$ is a reference cost and $\lambda\ge0$ controls the
trade-off between task success and model-call cost. The objective is a
routing policy that maximizes the expected discounted return
\begin{equation}
  J(\pi)=\E_{\tau\sim p_{\pi}}\!\left[
    \sum_{t=1}^{T}\gamma^{t-1}r_t
  \right],
  \label{eq:routing-objective}
\end{equation}
with discount $\gamma\in(0,1]$ and $p_{\pi}$ the trajectory distribution
induced by $\pi$, the model responses, and the environment. Because each
model choice changes the generated response and hence the states encountered
at later turns, routing decisions are \key{sequentially coupled rather than
independent turn-level predictions}.

\section{Method: \ours{}}
\label{sec:method}

\ours{} jointly trains a deployable \textbf{Routing Value Estimator (RVE)}
and a training-only \textbf{Response-Informed Value Estimator (RIVE)} over a
shared state--model representation (Figure~\ref{fig:method-overview}); the
two differ only in the information available to their value heads.

\suppressfloats[t]
\begin{figure}[t]
  \centering
  \includegraphics[width=0.92\linewidth]{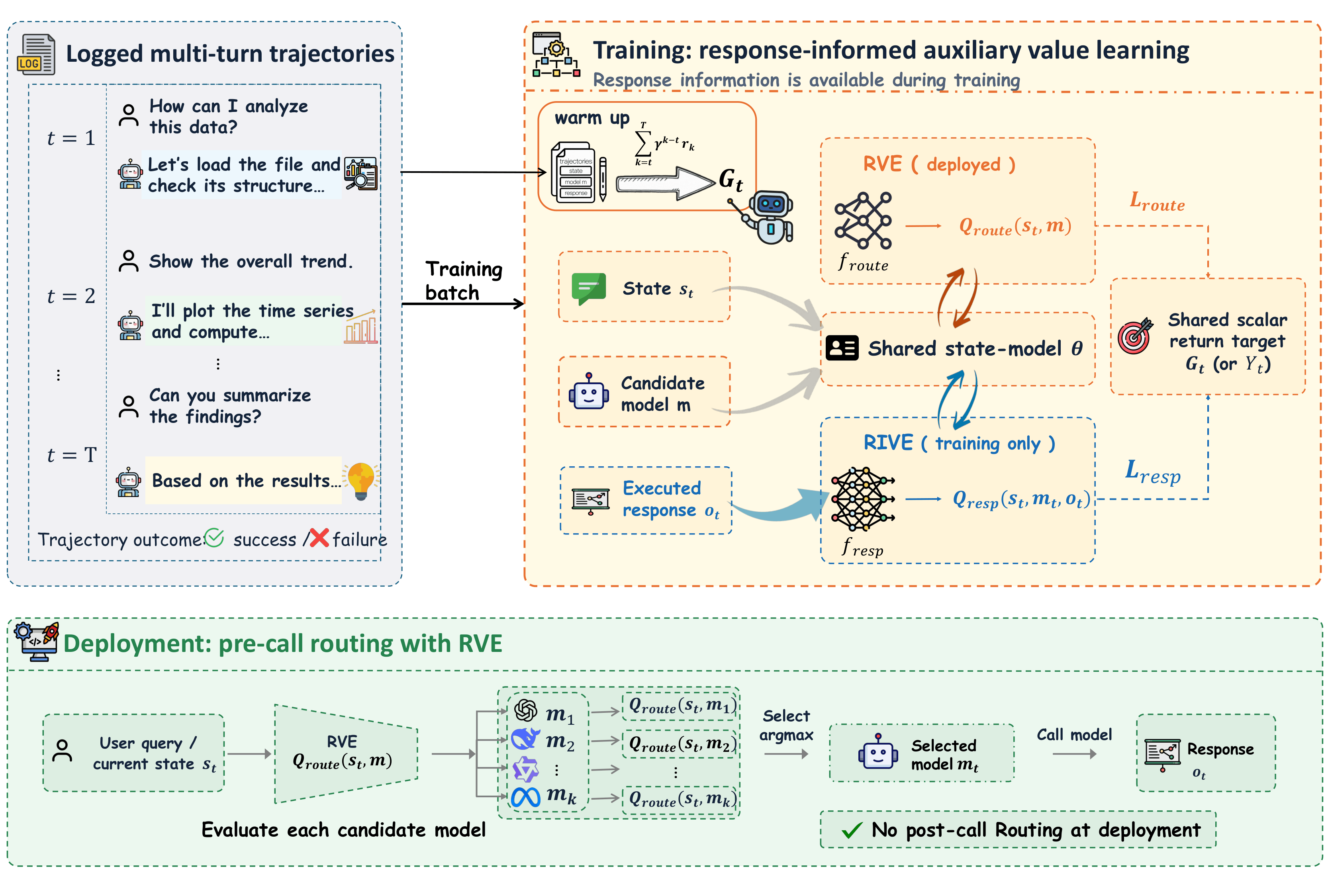}
  \caption{\ours{} separates pre-call deployment from post-call training.
  Both estimators read the same trainable state--model representation
  $z_\theta(s_t,m)$. The Routing Value Estimator (RVE, deployed) scores
  every candidate model from pre-call information and makes the routing
  decision; the Response-Informed Value Estimator (RIVE, training only)
  additionally conditions on the executed response and shapes the shared
  representation through its own regression loss. Both estimators regress to
  the same scalar return target; deployment uses RVE alone.}
  \label{fig:method-overview}
\end{figure}

\subsection{Dual value estimators with a shared state--model representation}
\label{sec:rve-rive-architecture}

\paragraph{Shared state--model representation.}
\ours{} represents each candidate action with a shared trainable
state--model encoder,
\begin{equation}
  z_\theta(s_t,m)
  =a_\theta\!\left(h_t,\;g_\theta(e_m,x_m)\right),
  \label{eq:shared-representation}
\end{equation}
which jointly encodes the current interaction state and the candidate model.
A frozen text encoder maps the serialized state $s_t$ to $h_t$; the
serialization omits earlier model identities. A trainable model-card
projector $g_\theta$ maps each model's text embedding $e_m$ and structured
fields $x_m$ to a model vector, and a trainable state adapter $a_\theta$
fuses it with $h_t$. The same representation feeds both value estimators, so
training signals from pre-call and response-informed prediction both shape
the representation used at deployment.

\paragraph{Routing Value Estimator (RVE).}
Given the shared state--model representation, RVE estimates the pre-call
value of selecting each candidate model,
\begin{equation}
  \Qroute(s_t,m)=\froute\!\left(z_\theta(s_t,m)\right),
  \label{eq:rve}
\end{equation}
and directly determines the deployment-time routing decision,
\begin{equation}
  m_t=\arg\max_{m\in\mathcal{M}}\Qroute(s_t,m).
  \label{eq:deployment-policy}
\end{equation}

\paragraph{Response-Informed Value Estimator (RIVE).}
During training, RIVE augments the shared state--model representation of the
executed model $m_t$ with the realized response $o_t$. Let $e(o_t)$ denote
the response representation, obtained by passing $o_t$ through the frozen
encoder and a trainable response projector. RIVE predicts the same return
target,
\begin{equation}
  \Qresp(s_t,m_t,o_t)=\fresp\!\left(z_\theta(s_t,m_t),\,e(o_t)\right).
  \label{eq:rive}
\end{equation}
\key{RIVE is discarded at deployment and adds no model
call or forward pass at inference time.}
Appendix~\ref{app:architecture-details} specifies the encodings and network
dimensions.

\subsection{Response-informed off-policy learning}
\label{sec:response-learning}
\label{sec:mc-initialization}
\label{sec:td-cql}

\ours{} learns the shared state--model representation in three stages:
it \key{first} grounds both value estimators with completed-trajectory returns,
\key{then} performs off-policy Bellman refinement over logged replay, and \key{finally}
adapts the routing policy with limited online interaction.

\paragraph{Response-informed joint learning.}
Both estimators regress to the same scalar return target. Let
$q_{\mathrm{route},t}=\Qroute(s_t,m_t)$ and
$q_{\mathrm{resp},t}=\Qresp(s_t,m_t,o_t)$, let $\ell_\delta$ denote the
Huber loss, and let $\E_t$ average over the current replay distribution.
For a scalar target $Z_t$, we optimize
\begin{equation}
  \mathcal{L}_{\mathrm{fit}}(Z)
  =
  \E_t\!\left[
    \ell_\delta(q_{\mathrm{route},t},Z_t)
    +
    \beta\,\ell_\delta(q_{\mathrm{resp},t},Z_t)
  \right],
  \label{eq:joint-regression}
\end{equation}
where $\beta\ge0$ controls the strength of the response-informed auxiliary
objective. For a fixed target, we refer to the first and second unweighted
expected loss terms as $\Lroute$ and $\Lresp$, respectively. Because both terms back-propagate through
$z_\theta(s_t,m_t)$, the realized response can shape the state--model
representation ultimately used by RVE.
Section~\ref{sec:theory} analyzes this mechanism.

\paragraph{Return-grounded Monte Carlo warm start.}
Let $G_t=\sum_{k=t}^{T}\gamma^{k-t}r_k$ denote the realized discounted return
from turn $t$ onward, with $r_k$ as in Equation~\ref{eq:immediate-reward}. We
initialize both value estimators with
$\mathcal{L}_{\mathrm{MC}}=\mathcal{L}_{\mathrm{fit}}(G)$.
Completed trajectories thus provide direct return supervision before
bootstrapping is introduced.

\paragraph{Off-policy Bellman refinement.}
After the Monte Carlo warm start, we perform temporal-difference learning
with Double DQN~\citep{vanhasselt2016deep}. Let $\omega$ collect the online
RVE parameters, including $\theta$, and let $\bar{\omega}$ denote their
target-network copy. RVE selects
the continuation action and the target network evaluates it:
\begin{align}
  &m_{t+1}^{\star}
  =
  \arg\max_{m'\in\mathcal{M}}
  \Qroute(s_{t+1},m';\omega),
  \label{eq:td-action}\\
  &Y_t
  =
  r_t+\gamma(1-d_t)
  \Qroute(s_{t+1},m_{t+1}^{\star};\bar{\omega}).
  \label{eq:td-target}
\end{align}
RVE defines the continuation because it is the estimator used for routing at
deployment, while both RVE and RIVE regress to the same Bellman target through
$\mathcal{L}_{\mathrm{fit}}(Y)$.

Because each logged state contains only one executed action whereas RVE must
rank the full candidate set at deployment, we additionally apply conservative
Q-learning (CQL)~\citep{kumar2020conservative}:
\begin{equation}
  \mathcal{L}_{\mathrm{CQL}}
  =
  \E_t\!\left[
    \log\sum_{m\in\mathcal{M}}
    \exp\Qroute(s_t,m)
    -
    q_{\mathrm{route},t}
  \right],
  \label{eq:cql}
\end{equation}
which suppresses optimistic values for weakly supported actions. The
off-policy objective is therefore
\begin{equation}
  \mathcal{L}_{\mathrm{TD}}
  =
  \mathcal{L}_{\mathrm{fit}}(Y)
  +
  \alpha\mathcal{L}_{\mathrm{CQL}},
  \label{eq:td-loss}
\end{equation}
where $\alpha$ controls the strength of conservative regularization.

\paragraph{Replay-anchored online adaptation.}
Offline training consists of the Monte Carlo warm start followed by
off-policy Bellman refinement. We then collect a limited number of
trajectories under the learned router and continue optimizing
$\mathcal{L}_{\mathrm{TD}}$ on a balanced mixture of logged and newly
collected transitions. Logged replay anchors learning, while fresh trajectories expose the router
to states induced by its own decisions. Appendix~\ref{app:optimization-protocol} provides the
optimization and stopping details, including the default $\beta=1$.

\subsection{Why response-informed auxiliary learning helps}
\label{sec:theory}

We show that the realized response makes RIVE a cleaner auxiliary
value-prediction problem and that, through the shared representation
$z_\theta(s,m)$, this cleaner signal improves the estimation of the deployed
RVE. The analysis covers the joint regression component of the objective for
a fixed target; CQL separately regularizes action ranking. Assumptions and
complete proofs are in Appendix~\ref{app:theory}.

\paragraph{Response-resolved uncertainty.}
Fix a turn and let $X=(S_t,M_t)$ denote the pre-call information, $O=O_t$
the realized response, and $Y$ the scalar value target used in
Section~\ref{sec:response-learning}. Define
\[
  \Qroute^{\star}(X)=\E[Y\mid X],
  \qquad
  \Qresp^{\star}(X,O)=\E[Y\mid X,O].
\]
We decompose the return target as
\begin{equation}
  Y
  =
  \Qroute^{\star}(X)
  +
  \Delta(X,O)
  +
  \epsilon,
  \label{eq:decomp}
\end{equation}
where
\[
  \Delta(X,O)
  =
  \Qresp^{\star}(X,O)
  -
  \Qroute^{\star}(X),
  \qquad
  \epsilon
  =
  Y-\Qresp^{\star}(X,O).
\]
Here $\Delta$ captures return variation that becomes predictable after
observing the response, while $\epsilon$ denotes the remaining uncertainty.

\begin{proposition}[Response-resolved uncertainty]
\label{prop:var}
For every pre-call input $X$,
\begin{equation}
  \Var(Y\mid X)
  =
  \Var(\Delta\mid X)
  +
  \E[\Var(Y\mid X,O)\mid X].
  \label{eq:var}
\end{equation}
\end{proposition}

\paragraph{Interpretation.}
RVE predicts the return from $(s,m)$ alone and must treat both $\Delta$ and
$\epsilon$ as residual variation; RIVE observes the response and explains
$\Delta$, leaving only $\epsilon$. The response thus \key{converts a
specific component of return variation from latent noise into an observed
covariate}, making RIVE a cleaner auxiliary regression problem.

\paragraph{Benefit to the shared representation.}
We next ask whether this cleaner auxiliary problem improves the
representation used by RVE. Around the population optimum, let $J(X)$
denote the sensitivity of the pre-call prediction to the shared
representation parameters, and define
\[
  H=\E[J(X)J(X)^{\top}],\qquad
  \Sigma_\Delta=\E[J(X)J(X)^{\top}\Delta^{2}],\qquad
  \Sigma_\epsilon=\E[J(X)J(X)^{\top}\epsilon^{2}].
\]
Under the local shared-signal and regularity conditions stated in
Appendix~\ref{app:theory}, we obtain the following result.

\begin{theorem}[Response-informed auxiliary learning reduces representation estimation error]
\label{thm:aux}
Let $\hat\theta_\beta$ denote the shared-representation parameters learned
with $\Lroute+\beta\Lresp$. Then
\begin{equation}
  \operatorname{Cov}(\hat\theta_\beta)
  =
  \frac{1}{n}
  H^{-1}
  \left(
    \Sigma_\epsilon
    +
    \frac{\Sigma_\Delta}{(1+\beta)^{2}}
  \right)
  H^{-1}
  +
  o(n^{-1}).
  \label{eq:cov-beta}
\end{equation}
Relative to RVE-only training ($\beta=0$), response-informed auxiliary
learning therefore reduces the leading contribution of response-induced
variation to the representation-estimation error by the factor
$1/(1+\beta)^{2}$. Whenever $\Sigma_\Delta\neq0$, the improvement is strict
along directions containing response-resolved variation.
\end{theorem}

\paragraph{Interpretation.}
RIVE benefits the deployed router without changing its decision rule: it
supplies an additional estimating signal for the same representation, free of
$\Delta$ as residual noise, so joint learning \key{suppresses the effect of
this variation on representation estimation by the factor
$1/(1+\beta)^{2}$}, while the contribution of $\epsilon$ is unchanged.

\section{Experiments}
\label{sec:experiments}
\raggedbottom

Our experiments ask whether \ours{} delivers a better accuracy--cost
trade-off in multi-turn interaction and what produces it: in-distribution and
out-of-distribution (OOD) performance (Section~\ref{sec:main-results}), the
effect of the cost coefficient and online adaptation
(Section~\ref{sec:online-adaptation}), component ablations that isolate the
response input (Section~\ref{sec:response-mechanism}), and turn-level
routing interventions (Section~\ref{sec:routing-interventions}).

\subsection{Experimental setup}
\label{sec:experimental-setup}

\paragraph{Datasets.}
We train and evaluate on the MINT Reasoning subset~\citep{wang2024mint},
which draws from GSM8K, HotpotQA, MATH, MMLU, and TheoremQA, using fixed,
disjoint training, validation, and test manifests. For transfer beyond the
training distribution we test on MuSiQue~\citep{trivedi2022musique},
DROP~\citep{dua2019drop}, and the mathematics, physics, and chemistry subsets
of MMLU-Pro~\citep{wang2024mmlupro}. Appendix~\ref{app:data-collection}
details domains, sampling, and trajectory collection.

\paragraph{Candidate models and baselines.}
The candidate pool contains DeepSeek V4-Pro, DeepSeek V4-Flash, Qwen3.7 Plus,
Qwen3.5-9B, Meta Llama-3.3-70B-Instruct, and OpenAI GPT-4o-mini, spanning a
range of capability and price. We compare against (i) the strongest fixed
model, DeepSeek V4-Pro; (ii) the single-turn routers Weighted kNN
(provided by ORBIT, \citealp{lai2026orbit}),
RouterDC~\citep{chen2024routerdc}, EmbedLLM~\citep{zhuang2025embedllm}, and
Avengers-Pro~\citep{zhang2025avengerspro}, which score query--model
compatibility from pre-call information without modeling long-horizon
consequences; and (iii) the multi-turn routers Random Router and
MTRouter~\citep{zhang2026mtrouter}, our closest baseline, which selects
models per turn under a cost-aware objective.

\paragraph{Evaluation metrics and protocol.}
We report accuracy on MINT and MMLU-Pro and exact match on MuSiQue and DROP;
model-call cost is measured over the full evaluation set relative to
each table's reference policy. Hyperparameters and checkpoints are
selected on the validation split, and deployment uses RVE alone.
Appendix~\ref{app:experimental-details} records optimization and stopping
settings.

\subsection{Accuracy--cost performance and OOD generalization}
\label{sec:main-results}

Table~\ref{tab:main_results} reports the main comparison. On MINT, \ours{}
improves accuracy over MTRouter while reducing model-call cost by 74.7\%,
and over the strongest fixed model while reducing cost by 80.1\%: stronger
task performance does not require allocating expensive models uniformly
across the trajectory. \ours{} thus occupies a different accuracy--cost
regime from both fixed-model and single-turn baselines: \finding{it reaches
the highest accuracy among routers at a cost comparable to the cheapest of
them}, reflecting a better allocation of capability across interaction states
rather than a globally cheaper model.

\begin{table}[H]
\centering
\caption{In-distribution and OOD performance. MINT Test and MMLU-Pro report
accuracy; MuSiQue and DROP report exact match. Cost is total evaluation API
cost relative to DeepSeek V4-Pro (1.000) on the same dataset; annotations
give score changes in points and relative cost changes from it. Bold marks
the best router per column.}
\label{tab:main_results}
\renewcommand{\arraystretch}{1.3}
\setlength{\tabcolsep}{3.5pt}
\resizebox{\textwidth}{!}{%
\begin{tabular}{@{}l@{\tablevsep}l@{\tablevsep}
cc@{\tablevsep}cc@{\tablevsep}cc@{\tablevsep}cc@{}}
\toprule
\multicolumn{1}{l}{\textbf{Type}} & \multicolumn{1}{l}{\textbf{Method}} &
\multicolumn{2}{c}{\textbf{MINT Test}} &
\multicolumn{2}{c}{\textbf{MuSiQue}} &
\multicolumn{2}{c}{\textbf{DROP}} &
\multicolumn{2}{c}{\textbf{MMLU-Pro}} \\
\cmidrule(lr){3-4}\cmidrule(lr){5-6}\cmidrule(lr){7-8}\cmidrule(lr){9-10}
\multicolumn{1}{l}{} & &
\textbf{ACC $\uparrow$} & \textbf{Cost $\downarrow$} &
\textbf{EM $\uparrow$} & \textbf{Cost $\downarrow$} &
\textbf{EM $\uparrow$} & \textbf{Cost $\downarrow$} &
\textbf{ACC $\uparrow$} & \textbf{Cost $\downarrow$} \\
\midrule
\multicolumn{1}{l@{\tablevsep}}{\cellcolor{groupgray}Model} &
\multicolumn{1}{l@{\tablevsep}}{\cellcolor{groupgray}DeepSeek V4-Pro} &
\cellcolor{groupgray}84.40 &
\multicolumn{1}{c@{\tablevsep}}{\cellcolor{groupgray}1.000} &
\cellcolor{groupgray}55.56 &
\multicolumn{1}{c@{\tablevsep}}{\cellcolor{groupgray}1.000} &
\cellcolor{groupgray}71.88 &
\multicolumn{1}{c@{\tablevsep}}{\cellcolor{groupgray}1.000} &
\cellcolor{groupgray}90.56 & \cellcolor{groupgray}1.000 \\
\midrule
\multirow{4}{*}{Single-turn router}
& \multicolumn{1}{l@{\tablevsep}}{RouterDC} & \metricdown{79.43}{4.97} & \metricdown{0.334}{66.6\%} &
\metricdown{50.56}{5.00} & \metricdown{0.202}{79.8\%} &
\metricdown{70.00}{1.88} & \metricdown{0.224}{77.6\%} &
\metricdown{87.22}{3.34} & \metricdown{0.211}{78.9\%} \\
& \multicolumn{1}{l@{\tablevsep}}{EmbedLLM} & \metricdown{82.98}{1.42} & \metricdown{0.220}{78.0\%} &
\metricdown{50.00}{5.56} & \metricdown{0.588}{41.2\%} &
\metricdown{\best{71.25}}{0.63} & \metricdown{0.413}{58.7\%} &
\metricdown{89.94}{0.62} & \metricup{1.358}{35.8\%} \\
& \multicolumn{1}{l@{\tablevsep}}{Avengers-Pro} & \metricup{85.82}{1.42} &
\metricdown{\best{0.133}}{86.7\%} & \metricdown{49.44}{6.12} &
\metricdown{0.217}{78.3\%} & \metricdown{70.63}{1.25} &
\metricdown{0.188}{81.2\%} &
\metricdown{87.78}{2.78} & \metricdown{0.226}{77.4\%} \\
& \multicolumn{1}{l@{\tablevsep}}{Weighted kNN} & \metricup{85.82}{1.42} & \metricup{1.370}{37.0\%} &
\metricdown{52.20}{3.36} & \metricup{1.388}{38.8\%} &
\metricdown{\best{71.25}}{0.63} & \metricdown{0.997}{0.3\%} &
\metricdown{\best{90.00}}{0.56} & \metricup{1.646}{64.6\%} \\
\midrule
\multirow{4}{*}{Multi-turn router}
& \multicolumn{1}{l@{\tablevsep}}{Random Router} & \metricdown{78.72}{5.68} & \metricdown{0.404}{59.6\%} &
\metricdown{46.10}{9.46} & \metricdown{0.272}{72.8\%} &
\metricdown{68.75}{3.13} & \metricdown{0.226}{77.4\%} &
\metricdown{84.44}{6.12} & \metricdown{0.276}{72.4\%} \\
& \multicolumn{1}{l@{\tablevsep}}{MTRouter} & \metricup{85.11}{0.71} & \metricdown{0.788}{21.2\%} &
\metricdown{48.33}{7.23} & \metricup{1.259}{25.9\%} &
\metricdown{70.63}{1.25} & \metricdown{0.987}{1.3\%} &
\metricdown{\best{90.00}}{0.56} & \metricup{1.204}{20.4\%} \\
& \multicolumn{1}{l@{\tablevsep}}{\offlineicon\ours{} Offline}
& \metricup{85.82}{1.42}
& \metricdown{0.281}{71.9\%}
& \metricdown{\best{55.00}}{0.56}
& \metricdown{0.407}{59.3\%}
& \metricdown{70.63}{1.25}
& \metricdown{0.171}{82.9\%}
& \metricdown{87.22}{3.34}
& \metricdown{0.177}{82.3\%} \\
& \multicolumn{1}{l@{\tablevsep}}{\onlineicon\ours{} Online}
& \metricup{\best{87.94}}{3.54}
& \metricdown{0.199}{80.1\%}
& \metricdown{51.11}{4.45}
& \metricdown{\best{0.190}}{81.0\%}
& \metricdown{\best{71.25}}{0.63}
& \metricdown{\best{0.107}}{89.3\%}
& \metricdown{\best{90.00}}{0.56}
& \metricdown{\best{0.092}}{90.8\%} \\
\bottomrule
\end{tabular}%
}
\end{table}

The cost advantage persists under distribution shift. Across the three OOD
benchmarks, \ours{} Online has the lowest cost among routers and matches or
exceeds MTRouter's task scores; on MMLU-Pro it matches MTRouter's accuracy at
92.4\% lower cost. \finding{The learned cost-allocation behavior is thus not
specific to the MINT training distribution.} The one exception is MuSiQue,
where online adaptation settles on a cheaper policy and gives up some exact
match; the offline router remains the most accurate router there.
Appendix~\ref{app:cases-other-routers} analyzes representative cases.

\FloatBarrier
\subsection{Cost control and online adaptation}
\label{sec:online-adaptation}

\paragraph{Cost-sensitive routing through $\lambda$.}
Figure~\ref{fig:lambda-accuracy-cost} sweeps the cost coefficient. Removing
the penalty ($\lambda=0$) raises model-call cost to 3.2$\times$ the default
and also lowers accuracy, so \finding{the cost term does more than enforce a
budget}: without it the router overuses expensive models in states where
additional capability contributes little. The default $\lambda=0.05$ gives
the best operating point; increasing $\lambda$ shifts the policy toward
cheaper models, and at the most cost-sensitive setting cost falls by 39.2\%
relative to the default while accuracy returns to MTRouter's level at a
fraction of its cost.

\begin{figure}[t]
  \centering
  \begin{minipage}[t]{0.49\textwidth}
    \vspace{0pt}
    \centering
    \includegraphics[height=0.56\linewidth]{%
      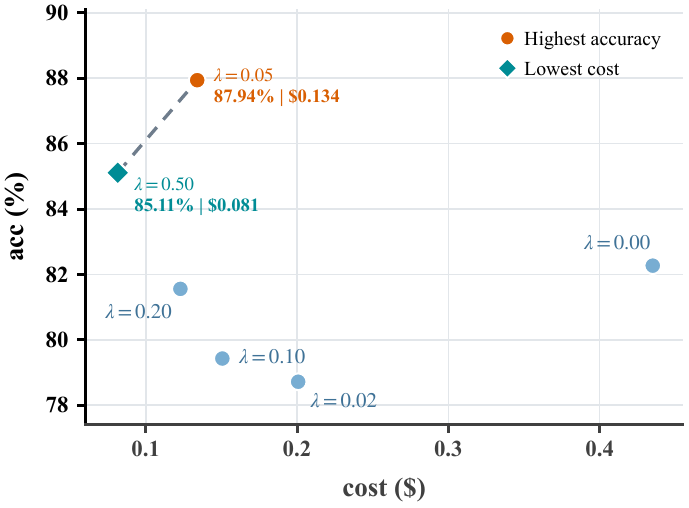}
    \caption{Cost-coefficient sensitivity on MINT Test-141. Highlighted settings
    achieve the highest accuracy ($\lambda=0.05$) and lowest cost
    ($\lambda=0.50$) among evaluated configurations. Cost denotes total
    test model-call cost in dollars.}
    \label{fig:lambda-accuracy-cost}
  \end{minipage}%
  \hfill
  \begin{minipage}[t]{0.49\textwidth}
    \vspace{0pt}
    \centering
    \includegraphics[height=0.56\linewidth]{%
      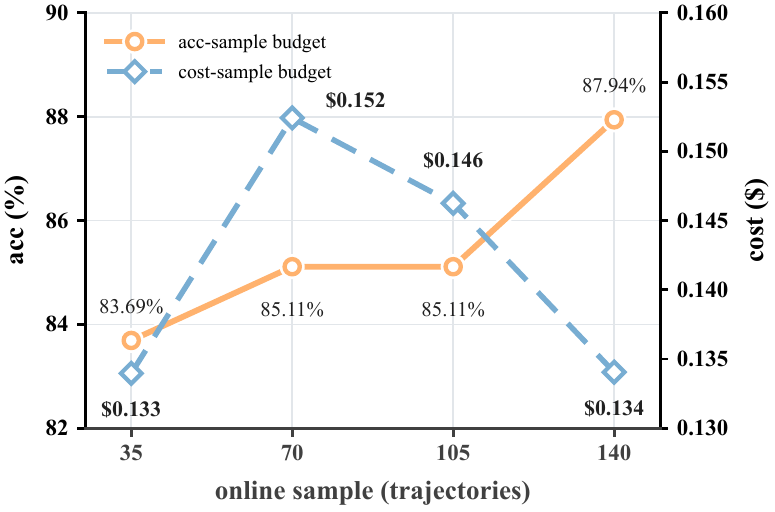}%
    \caption{Online adaptation on MINT Test-141. Accuracy (left axis) and
    total test model-call cost in dollars (right axis) after cumulatively
    adding 35, 70, 105, and 140 online trajectories in batches within a
    single training run.}
    \label{fig:online-budget-accuracy-cost}
  \end{minipage}
\end{figure}

\paragraph{Replay-anchored online adaptation.}
The offline router already exceeds MTRouter's MINT accuracy at lower cost
(Table~\ref{tab:main_results}), and limited online interaction refines it:
the full budget raises accuracy from 85.82\% to 87.94\% while reducing
test-time cost by 29.2\% relative to the offline router. \finding{The
improvement is thus not purchased through more expensive routing}; logged
trajectories supply most of the routing structure, and fresh trajectories
correct the mismatch between the offline behavior distribution and the
learned policy. Figure~\ref{fig:online-budget-accuracy-cost} tracks one run
as online trajectories are added in batches of 35: accuracy rises from
83.69\% after 35 to 87.94\% after 140, while final cost returns close to
its value at the first checkpoint.

\subsection{What drives the improvement?}
\label{sec:response-mechanism}

Table~\ref{tab:component_ablations} tests whether the gain comes from
response-informed auxiliary learning by removing RIVE, removing only its
response input, and removing CQL; all configurations deploy RVE alone.

\begin{table}[H]
\centering
\begin{threeparttable}
\caption{Component ablations on MINT Test-141. All configurations deploy RVE
alone; cost is relative to Ours.}
\label{tab:component_ablations}
\small
\renewcommand{\arraystretch}{1.25}
\setlength{\tabcolsep}{3.5pt}
\newcommand{\ablationcostplain}[1]{%
  \makebox[3.0em][r]{#1}\makebox[4.8em][l]{}}
\newcommand{\ablationcostup}[2]{%
  \makebox[3.0em][r]{#1}\makebox[4.8em][l]{\raisebox{-0.45ex}{%
  \hspace{0.15em}\scriptsize\textcolor{increasecolor}{%
  \ensuremath{\uparrow}\,#2}}}}
\begin{tabular*}{\textwidth}{@{\extracolsep{\fill}}l@{\tablevsep}ccc@{\tablevsep}cc@{}}
\toprule
\textbf{Configuration} & \textbf{RIVE} &
\shortstack{\textbf{Response}\\\textbf{input}} & \textbf{CQL} &
\textbf{ACC $\uparrow$} & \textbf{Cost $\downarrow$} \\
\midrule
\textbf{Ours} & $\checkmark$ & $\checkmark$ & $\checkmark$ &
\best{87.94\%} & \ablationcostplain{\best{1.000}} \\
\midrule
\hspace{1em}-\hspace{0.35em}w/o CQL & $\checkmark$ & $\checkmark$ &
\textcolor{red}{\ensuremath{\times}} & \metricdown{85.82\%}{2.12} &
\ablationcostup{2.163}{116.4\%} \\
\hspace{1em}-\hspace{0.35em}w/o RIVE & \textcolor{red}{\ensuremath{\times}} &
\textcolor{red}{\ensuremath{\times}} & $\checkmark$ &
\metricdown{80.14\%}{7.80} & \ablationcostup{1.172}{17.2\%} \\
\hspace{1em}-\hspace{0.35em}w/o response & $\checkmark$ &
\textcolor{red}{\ensuremath{\times}} & $\checkmark$ &
\metricdown{80.14\%}{7.80} & \ablationcostup{2.454}{145.4\%} \\
\bottomrule
\end{tabular*}
\begin{tablenotes}[flushleft]\footnotesize
\item Check: enabled component or access to the executed response; red
cross: ablated component or zeroed response input. w/o response keeps RIVE but
feeds a zero vector in place of the response. RIVE-enabled configurations use
$\beta=1$; removing RIVE sets $\beta=0$; other settings are fixed. Deltas
report ACC changes in points and cost changes from Ours (unrounded costs).
\end{tablenotes}
\end{threeparttable}
\end{table}

\paragraph{Role of the response-informed auxiliary estimator.}
Removing RIVE lowers accuracy by 7.80 points and raises cost, so its
training contributes to the deployed router. Keeping RIVE but zeroing its
response input produces exactly the same accuracy drop and a larger cost
increase. \finding{The improvement therefore does not come from adding a
second value head}: RIVE is useful because it observes the realized response.
This matches the zero-response case:
without the response, RIVE cannot explain $\Delta$, its residual equals
RVE's, and the joint objective supplies no additional estimating signal.
Appendix~\ref{app:cases-response-ablation} compares routing decisions of
\ours{} and this ablation on selected cases.

\paragraph{Role of conservative regularization.}
Removing CQL raises model-call cost by more than 100\% and lowers accuracy
by 2.12 points even with RIVE intact. The disproportionate cost increase
indicates that \finding{CQL mainly stabilizes how RVE ranks candidate actions
weakly supported by logged replay}, preventing optimistic values for
expensive models from dominating decisions.

\subsection{Why turn-level adaptation matters}
\label{sec:routing-interventions}

\begin{table}[t]
\centering
\small
\caption{Routing interventions on MINT evaluate retaining the first-turn model
in episodes where \ours{} originally switches and introducing a switch in
episodes where it originally uses one model.
Original: learned \ours{} policy.
Stay: retain the first-turn model throughout the episode.
Switch: select the second-highest RVE-scoring model at turn two, then resume
greedy routing. Cost is normalized to Original separately for each
intervention.}
\label{tab:full_vs_frozen}
\renewcommand{\arraystretch}{1.15}
\setlength{\tabcolsep}{5pt}
\begin{tabular*}{0.6\textwidth}{@{\extracolsep{\fill}}l@{\tablevsep}cc@{}}
\toprule
\textbf{Policy} & \textbf{ACC $\uparrow$} & \textbf{Cost $\downarrow$} \\
\midrule
Original & \best{84.75\%} & 1.000 \\
Stay & 72.88\% & \best{0.825} \\
\midrule
Original & 82.61\% & \best{1.000} \\
Switch & \best{84.78\%} & 1.429 \\
\bottomrule
\end{tabular*}
\end{table}
Table~\ref{tab:full_vs_frozen} examines both when to switch models and when
to stay. On the 59 MINT episodes in which \ours{} switches models, retaining
the first-turn model throughout lowers cost by 17.5\% but reduces successful
episodes from 50 to 43; since both policies start from the same model, the
difference isolates the later routing decisions. On 46 episodes in which
\ours{} keeps one model beyond the first turn, forcing the second-highest
RVE-scoring model at turn two (greedy routing thereafter) produces 78
switches instead of zero but only one additional success, 38 to 39, at
42.9\% higher cost.
\finding{Together, these results support switching models to improve task completion
and retaining the current model when the tested alternative offers limited
gains at higher cost.}

\flushbottom

\section{Conclusion}
\label{sec:conclusion}

\ours{} organizes multi-turn routing around when information becomes
available: RVE makes every deployment decision from the pre-call state, while
training-only RIVE conditions on the executed response and improves RVE only
through their shared representation. Conditioning on the response
removes response-induced return variation from the auxiliary regression
problem, and this cleaner signal reduces the leading asymptotic estimation
error of the deployed representation, using only responses and trajectories
already recorded during interaction and no additional inference-time path.
The current evaluation establishes a high-accuracy, lower-cost operating
point and an end-to-end difference between the full and RVE-only
configurations; cost-controlled response probes and matched transfer
ablations remain necessary for a narrower causal attribution to response
content.

\section*{Ethics Statement}
This work studies routing among existing language models in multi-turn
interaction, with the goal of reducing call cost while preserving task
performance. Cost-effective routing lowers the compute and monetary cost of
agentic systems and makes them accessible to a wider range of users. All
experiments use publicly available benchmarks and models and involve no
human subjects or personal data. Our method only selects which existing
model to call and does not generate new content or alter model behavior, so
it introduces no risks beyond those already present in the underlying
models. We are not aware of any foreseeable negative societal impact of this
work.

\bibliographystyle{iclr2027_conference}
\bibliography{references}

\appendix
\section{Theoretical Analysis: Proofs and Extensions}
\label{app:theory}

This appendix proves Proposition~\ref{prop:var} and Theorem~\ref{thm:aux}
and records three supporting facts that the main text invokes: the general
form of Theorem~\ref{thm:aux} without the shared-signal and centering
conditions (Appendix~\ref{app:theory-general}), the structure of the
response-resolved component $\Delta$ under an identity-blind environment
(Appendix~\ref{app:theory-structure}), the behavior of the decomposition under
bootstrapped targets (Appendix~\ref{app:theory-td}).

\paragraph{Setting.}
Fix a turn $t$ and a continuation policy $\pi$ that acts from turn $t+1$
onward. Write $\Ret_t=\sum_{j\ge t}\gamma^{j-t}r_j$ for the return from turn
$t$, with $r_j=-\lambda c_j/\costref+d_j y$. The regression target $Y$ is
$G_t$ in the Monte Carlo stage ($\pi=\mu$) and the bootstrapped target $Y_t$
of Equation~\ref{eq:td-target} in the temporal-difference stage, where the
local regression analysis holds the target network and selected continuation
actions fixed. We write
$X=(S_t,M_t)$, $O=O_t$, $\Qroute^{\star}(X)=\E[Y\mid X]$,
$\Qresp^{\star}(X,O)=\E[Y\mid X,O]$, related by the tower property
\begin{equation}
  \Qroute^{\star}(X)
  =\E_{O\mid X}\!\left[\Qresp^{\star}(X,O)\right],
  \label{eq:pre-post-relation}
\end{equation}
and use $\Delta$ and $\epsilon$ as in
Equation~\ref{eq:decomp}. Every statement below holds for any fixed $\pi$ and
any fixed target, so we omit the policy superscript. The superscript $\star$ on
$\Qroute$, $\Qresp$, and $\Ustar$ marks a population conditional quantity;
$m^{\star}$ and $V^{\star}$ denote a greedy action and its value. We use
$\sigfut$ as shorthand for $\Var(Y\mid X,O)$ when its arguments are clear.
Bounded returns ensure that all conditional expectations and variances exist.

\subsection{Proof of Proposition~\ref{prop:var}}
\label{app:proof-var}

\begin{proof}
By definition, $Y-\Qroute^{\star}(X)=\Delta(X,O)+\epsilon$. Since $\Delta$ is
measurable with respect to $(X,O)$ and $\E[\epsilon\mid X,O]=0$,
\begin{equation*}
  \E[\Delta\epsilon\mid X]
  =\E\!\left[\Delta\,\E[\epsilon\mid X,O]\,\middle|\,X\right]=0 .
\end{equation*}
Expanding the square and using $\E[\Delta\mid X]=0$,
\begin{equation*}
  \Var(Y\mid X)
  =\E\!\left[(\Delta+\epsilon)^{2}\mid X\right]
  =\Var(\Delta\mid X)+\E[\epsilon^{2}\mid X].
\end{equation*}
By the tower property,
$\E[\epsilon^{2}\mid X]=\E\!\left[\E[\epsilon^{2}\mid X,O]\mid X\right]
=\E\!\left[\Var(Y\mid X,O)\mid X\right]$, which is Equation~\ref{eq:var}.
The statements about the two estimators follow by reading the residual of
RVE about $\Qroute^{\star}$ as $\Delta+\epsilon$ and that of RIVE about
$\Qresp^{\star}$ as $\epsilon$. On a terminal turn,
$Y=G_t=-\lambda c(s,m,o)/\costref+y$; if the grader is a deterministic
function of $(s,o)$, then $Y$ is constant given $(X,O)$ and $\sigfut=0$.
\end{proof}

\subsection{Proof of Theorem~\ref{thm:aux}}
\label{app:proof-aux}

The proof relies on the following local model of the shared representation.

\begin{assumption}[Local shared-signal model]
\label{as:local}
(i) \emph{Correct specification.}
$\froute(z_{\theta^{\star}}(X))=\Qroute^{\star}(X)$ and
$\fresp(z_{\theta^{\star}}(X),e(O))=\Qresp^{\star}(X,O)$ at the population optimum
of the heads.
(ii) \emph{Shared pre-call signal.}
RIVE depends on $\theta$ only through the pre-call component:
$\nabla_\theta \fresp(z_{\theta^{\star}}(X),e(O))=J(X)$.
(iii) \emph{Centered response branch.}
The head of RIVE is parameterized so that its response-dependent part is
conditionally centered, $\E[\nabla_{\psi_{\mathrm{resp}}}\fresp\mid X]=0$ at the optimum.
(iv) $H\succ0$, and the standard regularity conditions for local
asymptotic analysis of M-estimators hold.
\end{assumption}

Condition (ii) says that $\Qresp=\Qroute+h$ locally, where the response branch $h$
absorbs $\Delta$; condition (iii) is the corresponding normalization
$\E[h\mid X]=0$ and makes the shared and response-specific parameters
asymptotically orthogonal. Appendix~\ref{app:theory-general} relaxes (ii)
and (iii).

\begin{proof}
Write the per-sample squared losses as
$\ell_{\mathrm{route}}(\theta)=\tfrac12\big(Y-\froute(z_\theta(X))\big)^{2}$ and
$\ell_{\mathrm{resp}}(\theta,\psi_{\mathrm{resp}})=\tfrac12\big(Y-\fresp(z_\theta(X),e(O);\psi_{\mathrm{resp}})\big)^{2}$,
and let $\widehat{\mathcal L}_\beta$ be the empirical average of
$\ell_{\mathrm{route}}+\beta\ell_{\mathrm{resp}}$ over $n$ i.i.d.\ samples. Under
Assumption~\ref{as:local}(i), the residuals at the population optimum are
\begin{equation*}
  Y-\froute(z_{\theta^{\star}}(X))=\Delta+\epsilon,
  \qquad
  Y-\fresp(z_{\theta^{\star}}(X),e(O);\psi_{\mathrm{resp}}^{\star})=\epsilon .
\end{equation*}

\emph{Scores.}
Using Assumption~\ref{as:local}(ii), the per-sample gradients with respect to
$\theta$ at the optimum are
\begin{equation*}
  g_{\mathrm{route}}=-J(X)(\Delta+\epsilon),
  \qquad
  g_{\mathrm{resp}}=-J(X)\,\epsilon,
  \qquad
  g_\beta=g_{\mathrm{route}}+\beta g_{\mathrm{resp}}=-J(X)\big[\Delta+(1+\beta)\epsilon\big].
\end{equation*}
Both scores have mean zero: $\E[J(X)\Delta]=\E[J(X)\E[\Delta\mid X]]=0$ and
$\E[J(X)\epsilon]=\E[J(X)\E[\epsilon\mid X,O]]=0$. The same conditioning
argument gives $\E[J(X)J(X)^{\top}\Delta\epsilon]=0$, hence
\begin{equation}
  \Var(g_\beta)
  =\E\!\left[J J^{\top}\big(\Delta+(1+\beta)\epsilon\big)^{2}\right]
  =\Sigma_\Delta+(1+\beta)^{2}\Sigma_\epsilon .
  \label{eq:score-var}
\end{equation}

\emph{Hessian.}
The population Hessian of $\E[\ell_{\mathrm{route}}]$ with respect to $\theta$ at
$\theta^{\star}$ is
$\E[JJ^{\top}]-\E[\nabla^{2}_\theta \froute\cdot(\Delta+\epsilon)]=H$, because
$\E[\Delta+\epsilon\mid X]=0$ removes the second term. The same argument with
residual $\epsilon$ gives $H$ for the $\theta$-block of $\E[\ell_{\mathrm{resp}}]$, so the
$\theta$-block of the joint Hessian is $H_\beta=(1+\beta)H$. The parameters
$\psi_{\mathrm{resp}}$ of the response head are estimated jointly with $\theta$, so the
joint Hessian also contains the cross block
$C=\E\!\left[J(X)\,\nabla_{\psi_{\mathrm{resp}}}\fresp^{\top}\right]$. Under
Assumption~\ref{as:local}(iii),
$C=\E\!\left[J(X)\,\E[\nabla_{\psi_{\mathrm{resp}}}\fresp\mid X]^{\top}\right]=0$, so the
joint Hessian is block diagonal in $(\theta,\psi_{\mathrm{resp}})$ and the $\theta$-block
can be analyzed on its own.

\emph{Sandwich.}
Under Assumption~\ref{as:local}(iv), standard local asymptotics for
M-estimators give
\begin{equation*}
  \operatorname{Cov}(\hat\theta_\beta)
  =\frac{1}{n}H_\beta^{-1}\Var(g_\beta)H_\beta^{-1}+o(n^{-1}).
\end{equation*}
Substituting $H_\beta=(1+\beta)H$ and Equation~\ref{eq:score-var},
\begin{align*}
  \operatorname{Cov}(\hat\theta_\beta)
  &=\frac{1}{n}\,\frac{1}{(1+\beta)^{2}}\,
    H^{-1}\!\left[\Sigma_\Delta+(1+\beta)^{2}\Sigma_\epsilon\right]\!H^{-1}
    +o(n^{-1})\\
  &=\frac{1}{n}H^{-1}\!\left[\Sigma_\epsilon
    +\frac{1}{(1+\beta)^{2}}\Sigma_\Delta\right]\!H^{-1}+o(n^{-1}),
\end{align*}
which is Equation~\ref{eq:cov-beta}. Setting $\beta=0$ gives the
RVE-only covariance, and the difference is
\begin{equation}
  \operatorname{Cov}(\hat\theta_0)-\operatorname{Cov}(\hat\theta_\beta)
  =\frac{1}{n}\!\left(1-\frac{1}{(1+\beta)^{2}}\right)\!
  H^{-1}\Sigma_\Delta H^{-1}+o(n^{-1})\;\succeq\;0,
  \label{eq:cov-gain}
\end{equation}
which is positive semidefinite because $\Sigma_\Delta\succeq0$, and is
positive definite on the range of $\Sigma_\Delta$.

\emph{Excess risk.}
Since $\nabla_\theta\Lroute(\theta^{\star})=\E[g_{\mathrm{route}}]=0$ and
$\nabla^{2}_\theta\Lroute(\theta^{\star})=H$, a second-order expansion
gives
\begin{equation*}
  \Lroute(\hat\theta_\beta)-\Lroute(\theta^{\star})
  =\tfrac12(\hat\theta_\beta-\theta^{\star})^{\top}H(\hat\theta_\beta-\theta^{\star})
  +o_p(n^{-1}).
\end{equation*}
Taking expectations and using
$\E[v^{\top}Hv]=\operatorname{Tr}[H\operatorname{Cov}(v)]$ with
Equation~\ref{eq:cov-beta} yields
\begin{equation}
  \E\!\left[\Lroute(\hat\theta_\beta)-\Lroute(\theta^{\star})\right]
  =\frac{1}{2n}\operatorname{Tr}\!\left[H^{-1}\!\left(
    \Sigma_\epsilon+\frac{\Sigma_\Delta}{(1+\beta)^{2}}
  \right)\right]+o(n^{-1}),
  \label{eq:excess-risk}
\end{equation}
which is nonincreasing in $\beta$ because $H^{-1}\Sigma_\Delta$ has nonnegative
trace.
\end{proof}

\subsection{General form without the shared-signal and centering conditions}
\label{app:theory-general}

Assumption~\ref{as:local}(ii)--(iii) idealize the implemented architecture,
in which RIVE has its own head and the two estimators share only the
state--model encoder. Dropping both conditions changes the constants but not
the direction of Theorem~\ref{thm:aux}. Let
$J_{\mathrm{route}}(X)=\nabla_\theta \froute$ and $J_{\mathrm{resp}}(X,O)=\nabla_\theta \fresp$ at the optimum,
and define $H_{\mathrm{route}}=\E[J_{\mathrm{route}}J_{\mathrm{route}}^{\top}]$, $H_{\mathrm{resp}}=\E[J_{\mathrm{resp}}J_{\mathrm{resp}}^{\top}]$,
$C=\E[J_{\mathrm{resp}}\nabla_{\psi_{\mathrm{resp}}}\fresp^{\top}]$, $H_{\psi\psi}=\E[\nabla_{\psi_{\mathrm{resp}}}\fresp\nabla_{\psi_{\mathrm{resp}}}\fresp^{\top}]$,
$\Sigma^{\mathrm{route}}_\Delta=\E[J_{\mathrm{route}}J_{\mathrm{route}}^{\top}\Delta^{2}]$,
$\Sigma^{\mathrm{route}}_\epsilon=\E[J_{\mathrm{route}}J_{\mathrm{route}}^{\top}\epsilon^{2}]$,
$\Sigma^{\mathrm{resp}}_\epsilon=\E[J_{\mathrm{resp}}J_{\mathrm{resp}}^{\top}\epsilon^{2}]$, and
$\Sigma^{\mathrm{route},\mathrm{resp}}_\epsilon=\E[J_{\mathrm{route}}J_{\mathrm{resp}}^{\top}\epsilon^{2}]$. The joint score with
respect to $\theta$ is $-J_{\mathrm{route}}(\Delta+\epsilon)-\beta J_{\mathrm{resp}}\epsilon$, and
profiling out $\psi_{\mathrm{resp}}$ replaces the RIVE information by its Schur
complement. The same sandwich argument then gives
\begin{equation}
\begin{aligned}
  \operatorname{Cov}(\hat\theta_\beta)
  &=\frac{1}{n}\,\tilde H_\beta^{-1}\!\left[
    \Sigma^{\mathrm{route}}_\Delta+\Sigma^{\mathrm{route}}_\epsilon
    +\beta\big(\Sigma^{\mathrm{route},\mathrm{resp}}_\epsilon+\Sigma^{\mathrm{route},\mathrm{resp}\,\top}_\epsilon\big)
    +\beta^{2}\Sigma^{\mathrm{resp}}_\epsilon
  \right]\!\tilde H_\beta^{-1}+o(n^{-1}),\\
  \tilde H_\beta
  &=H_{\mathrm{route}}+\beta\big(H_{\mathrm{resp}}-CH_{\psi\psi}^{-1}C^{\top}\big).
\end{aligned}
  \label{eq:cov-general}
\end{equation}
The term $\Sigma^{\mathrm{route}}_\Delta$ enters without any $\beta$-dependent factor in
the bracket and is divided by $\tilde H_\beta$ on both sides; since
$\tilde H_\beta\succeq H_{\mathrm{route}}$ with equality only when RIVE carries no
information about $\theta$, the response-induced contribution is attenuated by
a factor between $1/(1+\beta)^{2}$ and $1$, with the clean factor recovered
exactly when $J_{\mathrm{resp}}=J_{\mathrm{route}}$ and $C=0$. RIVE therefore never increases the
response-induced part of the estimation error, and the $\epsilon$-part is
reweighted rather than removed, matching the interpretation in the main
text.

\subsection{Structure of the response-resolved component}
\label{app:theory-structure}

Theorem~\ref{thm:aux} makes the gain of RIVE proportional to
$\Sigma_\Delta$. The following separability result, which uses the specific
structure of the routing environment, says what $\Delta$ is made of and hence
when $\Sigma_\Delta$ is large.
For this subsection, $Y=\Ret_t$ (and $G_t$ is its realized trajectory
sample); Appendix~\ref{app:theory-td} treats the bootstrapped target.

\begin{assumption}[Identity-blind environment]
\label{as:blind}
(i) Environment feedback, termination, and the task-success label depend on
the executed response but not on the identity of the model that produced it:
$P_{\mathrm{env}}(f_t\mid s_t,m_t,o_t)=P_{\mathrm{env}}(f_t\mid s_t,o_t)$ and
$y=y(s_T,o_T)$.
(ii) The call cost is a deterministic function of the call,
$c(s,m,o)=p^{\mathrm{in}}_m|s|+p^{\mathrm{out}}_m|o|$, where $p_m$ are the
per-token prices of model $m$.
(iii) The continuation policy reads only the next state $s_{t+1}$.
\end{assumption}

All three conditions hold in our setting: MINT tools and graders consume
response text only, API cost is price times token count, and the state
serialization of Section~\ref{sec:rve-rive-architecture} does not encode
earlier model identities.

\begin{proposition}[Post-call separability]
\label{prop:sep}
Under Assumption~\ref{as:blind}, for every $(s,m,o)$,
\begin{equation}
  \Qresp^{\star}(s,m,o)
  =\Ustar(s,o)-\lambda\,\frac{c(s,m,o)}{\costref},
  \quad\text{where}\quad
  \Ustar(s,o)\triangleq\E\!\left[
    d_t y+\gamma(1-d_t)V(S_{t+1})
    \,\middle|\,s,o
  \right],
  \label{eq:sep}
\end{equation}
with $V=V^{\pi}$ the value of the continuation policy
($V^{\star}(s')=\max_{m'}\Qroute^{\star}(s',m')$ under the greedy continuation).
Consequently,
\begin{equation}
  \Delta(s,m,o)
  =\Big(\Ustar(s,o)-\E_{o'\sim p_m(\cdot\mid s)}[\Ustar(s,o')]\Big)
  -\lambda\,\frac{c(s,m,o)-\E_{o'\sim p_m(\cdot\mid s)}[c(s,m,o')]}{\costref}.
  \label{eq:delta-structure}
\end{equation}
\end{proposition}

\begin{proof}
Write the return recursively as
$\Ret_t=r_t+\gamma(1-d_t)\Ret_{t+1}$ and condition on $(S_t,M_t,O_t)=(s,m,o)$.
By Assumption~\ref{as:blind}(ii) the cost term $-\lambda c(s,m,o)/\costref$ is a
constant given the conditioning event. By Assumption~\ref{as:blind}(i), the joint
conditional law of $(d_t,y,f_t)$ given $(s,m,o)$ equals its law given $(s,o)$.
The next state is the deterministic function
$S_{t+1}=(q,\mathcal H_{<t}\oplus(o,f_t))$ of $(s,o,f_t)$, so its conditional
law also depends on $(s,o)$ only. By Assumption~\ref{as:blind}(iii) the continuation
policy reads $S_{t+1}$ alone, hence
$\E[\Ret_{t+1}\mid S_{t+1},S_t,M_t,O_t]=\E[\Ret_{t+1}\mid S_{t+1}]=V(S_{t+1})$.
Taking conditional expectations gives Equation~\ref{eq:sep}.
Equation~\ref{eq:delta-structure} follows by subtracting the tower identity
$\Qroute^{\star}(s,m)=\E_{o'\sim p_m(\cdot\mid s)}[\Qresp^{\star}(s,m,o')]$
(Equation~\ref{eq:pre-post-relation}).
\end{proof}

\begin{remark}[Where RIVE helps]
Equation~\ref{eq:delta-structure} splits $\Delta$ into a progress term, the
deviation of the response-conditioned progress $\Ustar(s,o)$ from its mean
over the model's response distribution, and a cost term, the deviation of the
realized token cost from its mean. With $\lambda=0.05$ and normalized costs
of order one, the cost term is small, so $\Sigma_\Delta$ is dominated by
states in which different responses of the same model lead to different
downstream progress. These are precisely the situations that make single-turn
outcome labels ambiguous: when a strong model fails on an unfavorable
response realization, RVE alone attributes the failure to the model
representation $z_\theta(s,m)$, whereas RIVE attributes it to the
response content and leaves the shared representation to encode the model's
typical behavior. When both a cheap and an expensive model succeed, $\Delta$
is nearly zero for both and the two estimators agree.
\end{remark}

\subsection{Bootstrapped targets}
\label{app:theory-td}

In the temporal-difference stage both estimators regress to
$Y_t=r_t+\gamma(1-d_t)\Qroute(s_{t+1},m^{\star}_{t+1};\bar\omega)$ with a fixed target
network. Conditional on $(s,m,o)$, the only remaining randomness in $Y_t$ is
the environment feedback $f_t$ that determines $s_{t+1}$, so
\begin{equation*}
  \Var(Y_t\mid s,m,o)=\gamma^{2}(1-d_t)^{2}\,
  \Var_{f_t\sim P_{\mathrm{env}}(\cdot\mid s,o)}\!\left[\Qroute(s_{t+1},m^{\star}_{t+1};\bar\omega)\right],
\end{equation*}
which vanishes on terminal turns and whenever the tool feedback is a
deterministic function of $(s,o)$. In this stage, therefore, $\sigfut$ is
small or zero, and almost all of the target variance $\Var(Y_t\mid s,m)$ is
response-resolved: $\Delta$ contains the variation of the bootstrapped
next-state value across the responses the model could have produced. The
decomposition of Proposition~\ref{prop:var} and the covariance reduction of
Theorem~\ref{thm:aux} apply verbatim with $Y=Y_t$, and the bootstrapped stage
is the regime in which the response-resolved share of the target variance,
and hence the predicted benefit of RIVE, is largest. The reduction
concerns the estimation of $\theta$ for a fixed target; it does not by itself
bound the error introduced by bootstrapping from an imperfect target network,
which is the usual caveat of fitted value iteration and is outside the scope
of the local analysis.

\section{Additional experimental details and results}
\label{app:experimental-details}
\raggedbottom

\subsection{Datasets, trajectory collection, and manifests}
\label{app:data-collection}

Table~\ref{tab:dataset-distribution} reports the domain and sampled-to-source
counts for each benchmark composition. We use the MINT Reasoning subset,
MuSiQue-Ans dev, DROP validation, and MMLU-Pro test splits. The three sampled
MMLU-Pro subjects contain 3,782 examples within the full 12,032-example test
split.

\begingroup
\setlength{\intextsep}{4pt}
\setlength{\abovecaptionskip}{4pt}
\begin{table}[H]
\centering
\caption{Dataset domains and sampling distributions. The final column gives
the sampled count and source-pool size within each composition.}
\label{tab:dataset-distribution}
\footnotesize
\renewcommand{\arraystretch}{1.0}
\begin{tabular}{@{}l@{\tablevsep}l@{\tablevsep}l@{\tablevsep}c@{}}
\toprule
\textbf{Dataset} & \textbf{Domain} & \textbf{Composition} &
\textbf{Sample / Source size} \\
\midrule
\multirow{5}{*}{MINT Test} &
\multirow{5}{*}{Multi-task reasoning} &
GSM8K & 21 / 48 \\
\arrayrulecolor{black!20}\cline{3-4}\arrayrulecolor{black}
& & HotpotQA & 19 / 43 \\
\arrayrulecolor{black!20}\cline{3-4}\arrayrulecolor{black}
& & MATH & 45 / 100 \\
\arrayrulecolor{black!20}\cline{3-4}\arrayrulecolor{black}
& & MMLU & 34 / 76 \\
\arrayrulecolor{black!20}\cline{3-4}\arrayrulecolor{black}
& & TheoremQA & 22 / 49 \\
\midrule
\multirow{3}{*}{MuSiQue} &
\multirow{3}{*}{Multi-hop QA} &
2-hop & 60 / 1,252 \\
\arrayrulecolor{black!20}\cline{3-4}\arrayrulecolor{black}
& & 3-hop & 60 / 760 \\
\arrayrulecolor{black!20}\cline{3-4}\arrayrulecolor{black}
& & 4-hop & 60 / 405 \\
\midrule
\multirow{4}{*}{DROP} &
\multirow{4}{*}{Discrete text reasoning} &
Date & 40 / 157 \\
\arrayrulecolor{black!20}\cline{3-4}\arrayrulecolor{black}
& & Number & 40 / 5,850 \\
\arrayrulecolor{black!20}\cline{3-4}\arrayrulecolor{black}
& & Single-span & 40 / 2,962 \\
\arrayrulecolor{black!20}\cline{3-4}\arrayrulecolor{black}
& & Multi-span & 40 / 567 \\
\midrule
\multirow{3}{*}{MMLU-Pro} &
\multirow{3}{*}{STEM reasoning} &
Mathematics & 60 / 1,351 \\
\arrayrulecolor{black!20}\cline{3-4}\arrayrulecolor{black}
& & Physics & 60 / 1,299 \\
\arrayrulecolor{black!20}\cline{3-4}\arrayrulecolor{black}
& & Chemistry & 60 / 1,132 \\
\bottomrule
\end{tabular}
\end{table}
\endgroup

The MINT manifests contain 200 training, 35 validation, and 141 test tasks,
with at most ten turns per episode. Six fixed-model policies, one per
candidate model, and a Random Router each execute the 200 training tasks,
yielding 1,400 offline trajectories and 7,492 transitions. Online adaptation
uses the 140 reasoning tasks within the same training manifest and collects
at most 140 additional trajectories. \ours{} uses random seed 42, and
the trace-level analyses in Section~\ref{sec:routing-interventions} and
Appendix~\ref{app:additional-routing-analyses} use the corresponding canonical
run.

\subsection{Optimization and stopping protocol}
\label{app:optimization-protocol}

Except where noted, \ours{} and its variants use $\gamma=0.99$,
$\lambda=0.05$, $\beta=1$, and $\alpha=0.1$. Offline
optimization runs for five Monte Carlo epochs followed by 20
temporal-difference epochs. During online adaptation, we sample offline and
newly collected replay in equal proportions and anneal episode-indexed
$\epsilon$ from $0.10$ to $0.02$. We validate every 35 trajectories. After
140 trajectories, early stopping uses patience two and caps adaptation at
three passes through the task pool.

The Huber threshold is $\delta=1$. AdamW uses batch size 64, weight decay
$0.01$, and gradient clipping at 5. Offline training uses learning rate
$3\times10^{-4}$; online adaptation uses $3\times10^{-5}$ and performs one
gradient update per newly added transition. We synchronize the target
network every 200 optimization steps.

For the online sample study, we add trajectories in batches of 35 within a
single adaptation run. Figure~\ref{fig:online-budget-accuracy-cost} reports
test accuracy and total test model-call cost at cumulative sample counts of
35, 70, 105, and 140 trajectories. Each point evaluates the corresponding
router checkpoint on the same MINT Test-141 set, tracing the progression of
one training run as online data accumulate. In the cost-coefficient study,
validation selects the $\lambda=0$ checkpoint after 70 trajectories.

\subsection{State, response, and model-card encodings}
\label{app:architecture-details}

A frozen Qwen3-Embedding-0.6B encoder maps the serialized state $s_t$ to
$h_t\in\mathbb{R}^{1024}$. The serialization retains the task, the two latest
complete turns, and earlier turns up to 16,384 tokens; it omits the identities
of models used on earlier turns. Each candidate model has a frozen text
embedding $e_m$, normalized structured fields $x_m$ for context length,
capabilities, and prices, and a missing-value mask. A trainable model-card
projector produces
\begin{equation}
  z_m=g_{\theta}(e_m,x_m)\in\mathbb{R}^{128},
  \label{eq:model-card}
\end{equation}
and a trainable state adapter fuses $h_t$ with $z_m$ into the shared
state--model representation $z_\theta(s_t,m)$ of
Equation~\ref{eq:shared-representation}. RVE computes $\Qroute(s_t,m)$
from $z_\theta(s_t,m)$ for every candidate. RIVE passes the executed
response $o_t$ through the frozen encoder and a trainable response projector,
then computes $\Qresp(s_t,m_t,o_t)$ from $z_\theta(s_t,m_t)$ and the response
representation for the selected model. The two value heads have separate
parameters; the adapter and the projector are shared.

\subsection{Robustness to candidate-model unavailability}
\label{app:model-pool-robustness}

Table~\ref{tab:model_pool_robustness} tests routing stability under two
changes to the candidate model pool. We remove DeepSeek V4-Pro to test the
loss of a strong, high-cost option and Qwen3.5-9B to test the loss of a
weaker, low-cost option. Without Qwen3.5-9B, \ours{} matches MTRouter's
85.82\% accuracy while reducing cost by 74.4\%. Without DeepSeek V4-Pro,
\ours{} achieves 83.69\% accuracy at 72.3\% lower cost than MTRouter, with a
0.71-point accuracy gap. Across both conditions, \ours{} preserves a favorable
accuracy--cost operating point after the candidate pool loses either model.

\begin{table}[H]
\centering
\caption{Routing stability under candidate-model unavailability on MINT
Test-141. Each model pool removes the indicated model and uses \ours{} Online
as the reference. Colored annotations show accuracy changes from \ours{} Online
in percentage points and relative cost changes from \ours{} Online, whose cost
is set to 1.000. Bold values denote the best router result in each model pool.}
\label{tab:model_pool_robustness}
\renewcommand{\arraystretch}{1.55}
\setlength{\tabcolsep}{3.5pt}
\newcommand{\poolmetricplain}[1]{%
  \makebox[3.0em][r]{#1}\makebox[4.8em][l]{}}
\newcommand{\poolmetricup}[2]{%
  \makebox[3.0em][r]{#1}\makebox[4.8em][l]{\raisebox{-0.45ex}{%
  \hspace{0.15em}\scriptsize\textcolor{increasecolor}{%
  \ensuremath{\uparrow}\,#2}}}}
\newcommand{\poolmetricdown}[2]{%
  \makebox[3.0em][r]{#1}\makebox[4.8em][l]{\raisebox{-0.45ex}{%
  \hspace{0.15em}\scriptsize\textcolor{decreasecolor}{%
  \ensuremath{\downarrow}\,#2}}}}
\definecolor{poolequalcolor}{HTML}{808080}
\newcommand{\poolmetricequal}[1]{%
  \makebox[3.0em][r]{#1}\makebox[4.8em][l]{\raisebox{-0.45ex}{%
  \hspace{0.15em}\scriptsize\textcolor{poolequalcolor}{%
  \ensuremath{-\,0\%}}}}}
\resizebox{\textwidth}{!}{%
\begin{tabular}{@{}l@{\tablevsep}l@{\tablevsep}l@{\tablevsep}cc@{}}
\toprule
\multicolumn{1}{l}{\textbf{Model pool}} &
\multicolumn{1}{l}{\textbf{Type}} &
\multicolumn{1}{l}{\textbf{Method}} &
\multicolumn{2}{c}{\textbf{MINT Test}} \\
\cmidrule(lr){4-5}
\multicolumn{1}{l}{} & \multicolumn{1}{l}{} & &
\textbf{ACC $\uparrow$} & \textbf{Cost $\downarrow$} \\
\midrule
\multirow{6}{*}{w/o Qwen3.5-9B}
& \multirow{4}{*}{Single-turn router}
& RouterDC & \poolmetricdown{80.85}{4.97} &
\poolmetricup{2.527}{152.7\%} \\
& & EmbedLLM & \poolmetricdown{79.43}{6.39} &
\poolmetricdown{\best{0.820}}{18.0\%} \\
& & Avengers-Pro & \poolmetricdown{84.40}{1.42} &
\poolmetricup{2.887}{188.7\%} \\
& & Weighted kNN & \poolmetricdown{84.40}{1.42} &
\poolmetricup{3.002}{200.2\%} \\
\cmidrule{2-5}
& \multirow{2}{*}{Multi-turn router}
& MTRouter & \poolmetricequal{\best{85.82}} &
\poolmetricup{3.906}{290.6\%} \\
& & \onlineicon\ours{} Online & \poolmetricplain{\best{85.82}} &
\poolmetricplain{1.000} \\
\midrule
\multirow{6}{*}{w/o DeepSeek V4-Pro}
& \multirow{4}{*}{Single-turn router}
& RouterDC & \poolmetricdown{80.14}{3.55} &
\poolmetricup{2.527}{152.7\%} \\
& & EmbedLLM & \poolmetricup{\best{84.40}}{0.71} &
\poolmetricup{1.029}{2.9\%} \\
& & Avengers-Pro & \poolmetricequal{83.69} &
\poolmetricup{4.000}{300.0\%} \\
& & Weighted kNN & \poolmetricdown{82.97}{0.72} &
\poolmetricup{3.413}{241.3\%} \\
\cmidrule{2-5}
& \multirow{2}{*}{Multi-turn router}
& MTRouter & \poolmetricup{\best{84.40}}{0.71} &
\poolmetricup{3.610}{261.0\%} \\
& & \onlineicon\ours{} Online & \poolmetricplain{83.69} &
\poolmetricplain{\best{1.000}} \\
\bottomrule
\end{tabular}
}
\end{table}

\subsection{Routing after execution errors}
\label{app:additional-routing-analyses}

Table~\ref{tab:error-event-taxonomy} defines the error classes used in the
post-error analysis. The classifier matches error names and response strings
recorded in the interaction traces.

\begin{table}[H]
\centering
\caption{Error-event categories used in the post-error routing analysis.}
\label{tab:error-event-taxonomy}
\footnotesize
\renewcommand{\arraystretch}{1.35}
\setlength{\tabcolsep}{3pt}
\begin{tabular*}{\textwidth}{@{\extracolsep{\fill}}l@{\tablevsep}>{\raggedright\arraybackslash}p{0.70\textwidth}@{}}
\toprule
\textbf{Category} & \textbf{Matched error signatures} \\
\midrule
Non-executable action &
UnsafeCodeError; Invalid action; Illegal action; Unknown action; Action not allowed \\
\midrule
Formatting error &
SyntaxError; JSONDecodeError; Parse Error; FormatError; Malformed action \\
\midrule
Non-network tool failure &
TypeError; NameError; AttributeError; local-tool TimeoutError messages \\
\bottomrule
\end{tabular*}
\end{table}

The unit of analysis is an execution-error event at turn $t$. Error-event rate
divides classified error events by all recorded turns, including execution,
solution, and intermediate turns. Switch rate divides events for which the
model at $t+1$ differs from that at $t$ by error events with an observable next
turn. We exclude final-turn errors from this denominator; none occur in these
traces. Average switch cost is the mean cost of the changed-model call at
$t+1$. Next-turn recovery requires that this call produces a valid execution or
solution action without another classified error.

Across pooled canonical traces from MINT, MuSiQue, DROP, and MMLU-Pro,
\ours{} records an error-event rate of 3.96\%, compared with 8.48\% for
MTRouter. Following an error event, it switches models in 26.3\% of cases,
compared with 4.2\% for MTRouter, and spends 23.9\% less on the subsequent
switched call. The changed-model call recovers on the next turn in 93.3\% and
100.0\% of cases, respectively. These traces show that \ours{} combines
fewer recorded execution errors with more frequent model changes after an
error. The freezing intervention in Section~\ref{sec:routing-interventions}
separately tests the effect of preventing model changes.

\begin{table}[H]
\centering
\caption{Post-error routing behavior pooled across MINT Test, MuSiQue, DROP,
and MMLU-Pro canonical traces. The MTRouter switch cost is 1.000.}
\label{tab:routing-after-errors}
\footnotesize
\renewcommand{\arraystretch}{1.45}
\setlength{\tabcolsep}{3pt}
\begin{tabular*}{\textwidth}{@{\extracolsep{\fill}}l@{\tablevsep}cccc@{}}
\toprule
\textbf{Method} & \textbf{Error-event rate} & \textbf{Switch rate} &
\textbf{Switch cost $\downarrow$} & \textbf{Next-turn recovery} \\
\midrule
\ours{} Online & \best{3.96\%} & \best{26.3\%} &
\best{0.761} & 93.3\% \\
MTRouter & 8.48\% & 4.2\% & 1.000 & \best{100.0\%} \\
\bottomrule
\end{tabular*}
\end{table}

\subsection{Stage-wise model-selection preferences}
\label{app:stage-wise-routing-preferences}

We further examine whether each router changes its model-selection preference
across the semantic progression of an interaction. Rather than dividing an
episode by absolute turn index, we classify each recorded turn as
\emph{Analyze}, \emph{Solve}, or \emph{Submit} from its action and response
content. Table~\ref{tab:turn-stage-taxonomy} gives the classification rules.
Explicit submission markers define Submit turns; among the remaining turns,
tool use or substantive task-processing content defines Solve turns, while
planning and pre-solution intermediate output define Analyze turns.

\begin{table}[H]
\centering
\caption{Turn-stage definitions used in the stage-wise routing analysis.}
\label{tab:turn-stage-taxonomy}
\footnotesize
\renewcommand{\arraystretch}{1.35}
\setlength{\tabcolsep}{3pt}
\begin{tabular*}{\textwidth}{@{\extracolsep{\fill}}p{0.20\textwidth}@{\hspace{10pt}}>{\raggedright\arraybackslash}p{\dimexpr0.80\textwidth-10pt\relax}@{}}
\toprule
\textbf{Stage} & \textbf{Classification rule} \\
\midrule
Analyze (outer ring) &
Planning; task interpretation; task decomposition; blank pre-solution output
(no substantive computation or result). \\
\midrule
Solve (middle ring) &
Tool calls; code execution; tool feedback; calculations; evidence;
intermediate conclusions; verification. \\
\midrule
Submit (inner ring) &
Solution action; submitted solution field; forced final answer; solution tag. \\
\bottomrule
\end{tabular*}
\end{table}

Figure~\ref{fig:stage-routing-preferences} shows the resulting conditional
model-selection distributions. The two \ours{} variants select Qwen3.5-9B
most often during Analyze and Solve turns, whereas DeepSeek V4-Flash becomes
their most frequent choice during Submit turns. This shift indicates that the
learned policies allocate models according to the semantic role of a turn,
rather than following one fixed model mixture throughout an episode.

\begin{figure}[H]
\centering
\includegraphics[width=\textwidth]{%
  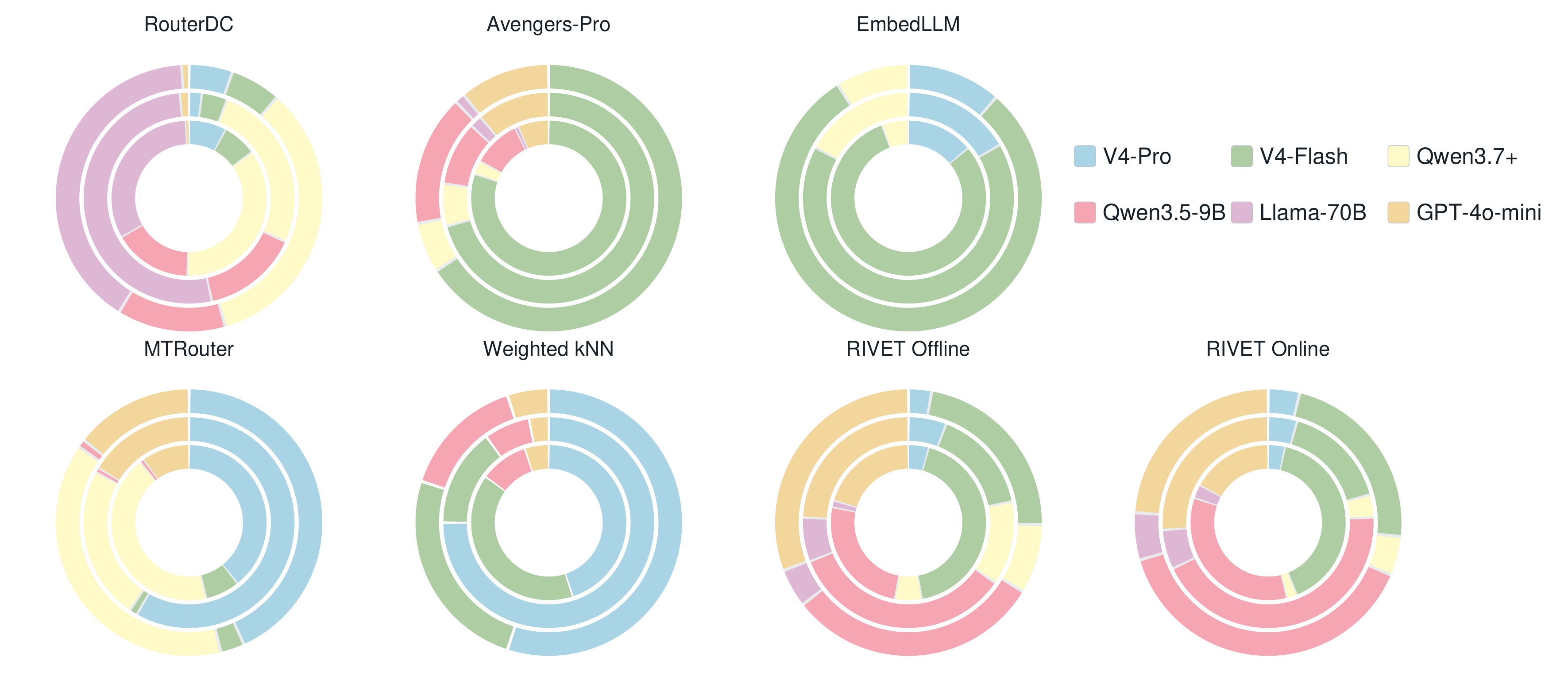}
\caption{Stage-wise model-selection preferences for seven routers on MINT
Test-141. Each panel
shows the distribution of selected models within Analyze (outer ring), Solve
(middle ring), and Submit (inner ring) turns; each ring is normalized
independently to 100\%. Segment colors denote candidate models.}
\label{fig:stage-routing-preferences}
\end{figure}

\begingroup
\raggedbottom
\section{Case studies}
\label{app:case-studies}
\definecolor{casebg}{HTML}{F1F4F9}
\definecolor{caseborder}{HTML}{CBD3DF}
\definecolor{ink}{HTML}{202B3D}
\definecolor{muted}{HTML}{697386}
\setlength{\parindent}{0pt}
\setlength{\parskip}{3pt}
\newcommand{\divider}{\par\vspace{2pt}{\color{caseborder}\hrule height .5pt}\vspace{4pt}}
\newcounter{caseexample}
\newtcolorbox{casebox}[1][3pt]{breakable, colback=casebg, colframe=caseborder, boxrule=.4pt, arc=0pt, boxsep=0pt, left=7pt, right=7pt, top=7pt, bottom=7pt, before upper={\small\setlength{\parskip}{#1}}}
\newcommand{\step}[2]{\textbf{#1}\newline #2}
\newcolumntype{L}[1]{>{\raggedright\arraybackslash}p{#1}}

{\small Across all cases, we condense tasks and actions, quote recorded responses and tool outputs verbatim, and report total episode costs in USD to six decimal places.\par}\vspace{8pt}
\subsection{Comparison with other routers}
\label{app:cases-other-routers}
The cases illustrate how \ours{} combines correct task completion with
lower episode cost. In Example~\ref{case:example-3}, \ours{} submits the
correct answer after two turns, whereas RouterDC obtains the correct
intermediate result but then introduces an unnecessary unit conversion and
returns an incorrect answer. In Example~\ref{case:example-2}, both \ours{}
and EmbedLLM answer correctly in two turns, but \ours{} uses DeepSeek
V4-Flash instead of DeepSeek V4-Pro, completing the task at lower cost.

The recovery trajectories also show how routers proceed after an execution
error. Table~\ref{tab:routing-after-errors} in
Appendix~\ref{app:additional-routing-analyses} reports post-error switch
rates of 26.3\% for \ours{} and 4.2\% for MTRouter.
In Example~\ref{case:example-6}, both routers initially select DeepSeek
V4-Pro and encounter the same disallowed-import error. From identical
histories, \ours{} switches to Qwen3.5-9B and answers correctly from the
supplied passages at turn two. MTRouter retains DeepSeek V4-Pro, corrects
the tool call, and continues searching until it submits the same answer at
turn eight. Both resolve the execution error on the next turn, but \ours{}
completes the task earlier at lower episode cost.

\begin{casebox}
\refstepcounter{caseexample}\label{case:example-1}
\textbf{Example \thecaseexample. Identifying the drummer of Maroon 5}\par
\textbf{Task.} Who is the current drummer of the band who did the song ``What Lovers Do''?
\divider\textbf{Turn-by-turn model choices and actions}\par
\begingroup\small\setlength{\tabcolsep}{5pt}\renewcommand{\arraystretch}{1.12}\arrayrulecolor{caseborder}
\begin{tabular}{@{}L{.055\linewidth}L{.44\linewidth}L{.44\linewidth}@{}}
\textbf{Turn} & \textbf{RIVET} & \textbf{MTRouter}\\\hline
1 & \step{OpenAI GPT-4o-mini}{Plans to identify the band.} & \step{OpenAI GPT-4o-mini}{Plans to identify the band.}\\[3pt]
2 & \step{OpenAI GPT-4o-mini}{Searches for the song.} & \step{OpenAI GPT-4o-mini}{Searches for the song.}\\[3pt]
3 & \step{OpenAI GPT-4o-mini}{Identifies Maroon 5 and plans a drummer search.} & \step{OpenAI GPT-4o-mini}{Identifies Maroon 5 and plans a drummer search.}\\[3pt]
4 & \step{OpenAI GPT-4o-mini}{Searches for the drummer; the result names Matt Flynn.} & \step{OpenAI GPT-4o-mini}{Searches for the drummer; the result names Matt Flynn.}\\[3pt]
5 & \step{OpenAI GPT-4o-mini}{States Matt Flynn and prepares to submit.} & \step{DeepSeek V4-Pro}{Performs an additional verification search.}\\[3pt]
6 & \step{OpenAI GPT-4o-mini}{Submits Matt Flynn.} & \step{OpenAI GPT-4o-mini}{Submits Matt Flynn.}\\[3pt]
\end{tabular}\endgroup
\divider\textbf{Selected responses and tool outputs}\par
\textbf{T4, RIVET / OpenAI GPT-4o-mini (tool output excerpt):}\par
{\small\raggedright drummer Matt Flynn\par}
\textbf{T5, RIVET / OpenAI GPT-4o-mini (response excerpt):}\par
{\small\raggedright The current drummer of Maroon 5 is Matt Flynn. I will now submit the final answer.\par}
\textbf{T5, MTRouter / DeepSeek V4-Pro (response excerpt):}\par
{\small\raggedright The Wikipedia summary states that Maroon 5's drummer is Matt Flynn. I will verify this with a more specific search to ensure there hasn't been a recent change.\par}
\divider\textbf{Final answers}\par
\textbf{RIVET:} Matt Flynn. \textbf{Correct.} 6 turns; Cost: \$0.000639.\par
\textbf{MTRouter:} Matt Flynn. \textbf{Correct.} 6 turns; Cost: \$0.002832.\par
\end{casebox}
{\footnotesize Task hotpotqa:33.\par}\vspace{10pt}
\begin{casebox}[1.5pt]
\refstepcounter{caseexample}\label{case:example-2}
\textbf{Example \thecaseexample. The final three digits of a power}\par
\textbf{Task.} What is the sum of the final three digits of the integer representation of $5^{100}$?
\divider\textbf{Turn-by-turn model choices and actions}\par
\begingroup\small\setlength{\tabcolsep}{5pt}\renewcommand{\arraystretch}{1.12}\arrayrulecolor{caseborder}
\begin{tabular}{@{}L{.055\linewidth}L{.44\linewidth}L{.44\linewidth}@{}}
\textbf{Turn} & \textbf{RIVET} & \textbf{EmbedLLM}\\\hline
1 & \step{DeepSeek V4-Flash}{Computes the last three digits (625) and their sum (13).} & \step{DeepSeek V4-Pro}{Computes the last three digits (625) and their sum (13).}\\[3pt]
2 & \step{DeepSeek V4-Flash}{Submits 13.} & \step{DeepSeek V4-Pro}{Submits 13.}\\[3pt]
\end{tabular}\endgroup
\divider\textbf{Selected responses and tool outputs}\par
\textbf{T1, RIVET / DeepSeek V4-Flash (response excerpt):}\par
{\small\raggedright print(sum(int(d) for d in str(pow(5, 100, 1000))))\par}
\textbf{T1, EmbedLLM / DeepSeek V4-Pro (response excerpt):}\par
{\small\raggedright mod = pow(5, 100, 1000)\par}
\textbf{T2, RIVET / DeepSeek V4-Flash (response excerpt):}\par
{\small\raggedright \textless{}solution\textgreater{}13\textless{}/solution\textgreater{}\par}
\textbf{T2, EmbedLLM / DeepSeek V4-Pro (response excerpt):}\par
{\small\raggedright \textless{}solution\textgreater{}13\textless{}/solution\textgreater{}\par}
\divider\textbf{Final answers}\par
\textbf{RIVET:} 13. \textbf{Correct.} 2 turns; Cost: \$0.000113.\par
\textbf{EmbedLLM:} 13. \textbf{Correct.} 2 turns; Cost: \$0.002088.\par
\end{casebox}
{\footnotesize Task math:73.\par}\vspace{10pt}
\begin{casebox}[1.5pt]
\refstepcounter{caseexample}\label{case:example-3}
\textbf{Example \thecaseexample. Buying snowshoes for six sled dogs}\par
\textbf{Task.} Mario has 6 sled dogs, each with four legs. A pair of snowshoes costs \$12. How much does it cost to buy snowshoes for all his dogs? Return an integer.
\divider\textbf{Turn-by-turn model choices and actions}\par
\begingroup\small\setlength{\tabcolsep}{5pt}\renewcommand{\arraystretch}{1.12}\arrayrulecolor{caseborder}
\begin{tabular}{@{}L{.055\linewidth}L{.44\linewidth}L{.44\linewidth}@{}}
\textbf{Turn} & \textbf{RIVET} & \textbf{RouterDC}\\\hline
1 & \step{Qwen3.5-9B}{Computes 12 pairs at 12 dollars per pair; tool output: 144.} & \step{Meta Llama-3.3-70B-Instruct}{Plans to calculate the required number of pairs.}\\[3pt]
2 & \step{DeepSeek V4-Flash}{Submits 144.} & \step{Meta Llama-3.3-70B-Instruct}{Computes the total; tool output: 144.0.}\\[3pt]
3 & \textcolor{muted}{Episode ended} & \step{Meta Llama-3.3-70B-Instruct}{Converts dollars to cents; tool output: 14400.}\\[3pt]
4 & \textcolor{muted}{Episode ended} & \step{Meta Llama-3.3-70B-Instruct}{Submits 14400.}\\[3pt]
\end{tabular}\endgroup
\divider\textbf{Selected responses and tool outputs}\par
\textbf{T1, RIVET / Qwen3.5-9B (response excerpt):}\par
{\small\raggedright total\_cost = num\_pairs * cost\_per\_pair\par}
\textbf{T3, RouterDC / Meta Llama-3.3-70B-Instruct (response excerpt):}\par
{\small\raggedright the output should be in cents to avoid losing precision due to the decimal point.\par}
\textbf{T3, RouterDC / Meta Llama-3.3-70B-Instruct (response excerpt):}\par
{\small\raggedright total\_cost\_cents = int(total\_cost\_dollars * 100)\par}
\textbf{T2, RIVET / DeepSeek V4-Flash (response excerpt):}\par
{\small\raggedright \textless{}solution\textgreater{}144\textless{}/solution\textgreater{}\par}
\divider\textbf{Final answers}\par
\textbf{RIVET:} 144. \textbf{Correct.} 2 turns; Cost: \$0.000179.\par
\textbf{RouterDC:} 14400. \textbf{Incorrect.} 4 turns; Cost: \$0.000403.\par
\end{casebox}
{\footnotesize Task gsm8k:17.\par}\vspace{10pt}
\begin{casebox}
\refstepcounter{caseexample}\label{case:example-4}
\textbf{Example \thecaseexample. Area defined by floor functions}\par
\textbf{Task.} What is the area of the region in the $xy$-plane satisfying $\lfloor x\rfloor\lfloor y\rfloor=16$?
\divider\textbf{Turn-by-turn model choices and actions}\par
\begingroup\small\setlength{\tabcolsep}{5pt}\renewcommand{\arraystretch}{1.12}\arrayrulecolor{caseborder}
\begin{tabular}{@{}L{.055\linewidth}L{.44\linewidth}L{.44\linewidth}@{}}
\textbf{Turn} & \textbf{RIVET} & \textbf{RouterDC}\\\hline
1 & \step{DeepSeek V4-Flash}{Submits 10.} & \step{Meta Llama-3.3-70B-Instruct}{Lists ten unit squares, then writes code that counts each positive ordered pair four times; makes no tool call.}\\[3pt]
2 & \textcolor{muted}{Episode ended} & \step{Meta Llama-3.3-70B-Instruct}{Revises the code to count 8, 8, and 4 squares for the three positive unordered pairs; makes no tool call.}\\[3pt]
3 & \textcolor{muted}{Episode ended} & \step{Meta Llama-3.3-70B-Instruct}{Repeats the code in a Markdown block; makes no tool call.}\\[3pt]
4 & \textcolor{muted}{Episode ended} & \step{Meta Llama-3.3-70B-Instruct}{Claims a code output of 24 and submits 24; the log contains no tool execution.}\\[3pt]
\end{tabular}\endgroup
\divider\textbf{Selected responses and tool outputs}\par
\textbf{T1, RIVET / DeepSeek V4-Flash (response excerpt):}\par
{\small\raggedright \textless{}solution\textgreater{}10\textless{}/solution\textgreater{}\par}
\textbf{T1, RouterDC / Meta Llama-3.3-70B-Instruct (response excerpt):}\par
{\small\raggedright The area of each unit square is 1. Since we have 10 such pairs (considering both positive and negative factors), we might initially think to calculate the area based on these pairs.\par}
\textbf{T4, RouterDC / Meta Llama-3.3-70B-Instruct (response excerpt):}\par
{\small\raggedright We've executed the code and now have the total area.\par}
\divider\textbf{Final answers}\par
\textbf{RIVET:} 10. \textbf{Correct.} 1 turn; Cost: \$0.000126.\par
\textbf{RouterDC:} 24. \textbf{Incorrect.} 4 turns; Cost: \$0.001318.\par
\end{casebox}
{\footnotesize Task math:0.\par}\vspace{10pt}
\begin{casebox}
\refstepcounter{caseexample}\label{case:example-5}
\textbf{Example \thecaseexample. Expected number of die rolls}\par
\textbf{Task.} Each morning Bob rolls a fair six-sided die, rerolling whenever he gets 1 and stopping on any other number. What is the expected total number of rolls in a non-leap year?
\divider\textbf{Turn-by-turn model choices and actions}\par
\begingroup\small\setlength{\tabcolsep}{5pt}\renewcommand{\arraystretch}{1.12}\arrayrulecolor{caseborder}
\begin{tabular}{@{}L{.055\linewidth}L{.44\linewidth}L{.44\linewidth}@{}}
\textbf{Turn} & \textbf{RIVET} & \textbf{Avengers-Pro}\\\hline
1 & \step{Qwen3.5-9B}{Computes the expected annual total; tool output: 438.0.} & \step{Qwen3.5-9B}{Computes the expected annual total; tool output: 438.0.}\\[3pt]
2 & \step{OpenAI GPT-4o-mini}{Switches to OpenAI GPT-4o-mini and submits 438.} & \step{Qwen3.5-9B}{Keeps Qwen3.5-9B and submits 438 with a derivation.}\\[3pt]
\end{tabular}\endgroup
\divider\textbf{Selected responses and tool outputs}\par
\textbf{T1, RIVET / Qwen3.5-9B (tool output excerpt):}\par
{\small\raggedright Expected rolls per day: 1.2\newline Total expected rolls in 365 days: 438.0\par}
\textbf{T1, Avengers-Pro / Qwen3.5-9B (tool output excerpt):}\par
{\small\raggedright Expected rolls per day: 1.2\newline Total expected rolls in 365 days: 438.0\par}
\textbf{T2, RIVET / OpenAI GPT-4o-mini (response excerpt):}\par
{\small\raggedright \textless{}solution\textgreater{}438\textless{}/solution\textgreater{}\par}
\textbf{T2, Avengers-Pro / Qwen3.5-9B (response excerpt):}\par
{\small\raggedright Expected rolls per day = 1/p = 6/5 = 1.2\par}
\divider\textbf{Final answers}\par
\textbf{RIVET:} 438. \textbf{Correct.} 2 turns; Cost: \$0.000283.\par
\textbf{Avengers-Pro:} D) 438. \textbf{Correct.} 2 turns; Cost: \$0.000437.\par
\end{casebox}
{\footnotesize Task mmlu:22.\par}\vspace{10pt}
\begin{casebox}
\refstepcounter{caseexample}\label{case:example-6}
\textbf{Example \thecaseexample. Answering after an invalid tool invocation}\par
\textbf{Task.} What is the average yearly income for a person of the same nationality as the creator of MacGruder and Loud?
\divider\textbf{Turn-by-turn model choices and actions}\par
\begingroup\small\setlength{\tabcolsep}{5pt}\renewcommand{\arraystretch}{1.12}\arrayrulecolor{caseborder}
\begin{tabular}{@{}L{.055\linewidth}L{.44\linewidth}L{.44\linewidth}@{}}
\textbf{Turn} & \textbf{RIVET} & \textbf{MTRouter}\\\hline
1 & \step{DeepSeek V4-Pro}{Imports wikipedia before calling the search function; the tool rejects the import.} & \step{DeepSeek V4-Pro}{Imports wikipedia before calling the search function; the tool rejects the import.}\\[3pt]
2 & \step{Qwen3.5-9B}{Switches to Qwen3.5-9B and submits the answer from the supplied paragraphs.} & \step{DeepSeek V4-Pro}{Removes the import and searches for MacGruder and Loud.}\\[3pt]
3 & \textcolor{muted}{Episode ended} & \step{DeepSeek V4-Pro}{Searches for the creator of MacGruder and Loud.}\\[3pt]
4 & \textcolor{muted}{Episode ended} & \step{DeepSeek V4-Pro}{Searches again for the television series creator.}\\[3pt]
5 & \textcolor{muted}{Episode ended} & \step{DeepSeek V4-Pro}{Searches for the television series.}\\[3pt]
6 & \textcolor{muted}{Episode ended} & \step{DeepSeek V4-Pro}{Searches for the cast.}\\[3pt]
7 & \textcolor{muted}{Episode ended} & \step{DeepSeek V4-Pro}{Searches for MacGruder and Loud and Aaron Spelling.}\\[3pt]
8 & \textcolor{muted}{Episode ended} & \step{DeepSeek V4-Pro}{Submits the answer with supporting paragraphs.}\\[3pt]
\end{tabular}\endgroup
\divider\textbf{Selected responses and tool outputs}\par
\textbf{T1, RIVET / DeepSeek V4-Pro (response excerpt):}\par
{\small\raggedright import wikipedia\newline print(wikipedia\_search("MacGruder and Loud"))\par}
\textbf{T1, RIVET / DeepSeek V4-Pro (tool output excerpt):}\par
{\small\raggedright UnsafeCodeError: Import is not allowed: wikipedia\par}
\textbf{T2, RIVET / Qwen3.5-9B (response excerpt):}\par
{\small\raggedright Answer: \$72,641\newline Support: 4, 8, 15\par}
\textbf{T2, MTRouter / DeepSeek V4-Pro (response excerpt):}\par
{\small\raggedright print(wikipedia\_search("MacGruder and Loud"))\par}
\textbf{T8, MTRouter / DeepSeek V4-Pro (response excerpt):}\par
{\small\raggedright Answer: \$72,641\newline Support: 15,8,4\par}
\divider\textbf{Final answers}\par
\textbf{RIVET:} \$72,641. \textbf{Correct.} 2 turns; Cost: \$0.003489.\par
\textbf{MTRouter:} \$72,641. \textbf{Correct.} 8 turns; Cost: \$0.015209.\par
\end{casebox}
{\footnotesize Task musique\_ans:3hop1\_\_694534\_160088\_73916.\par}\vspace{10pt}
\subsection{Comparison with response-ablated training}
\label{app:cases-response-ablation}
These cases illustrate how response-informed training affects accuracy and
cost through the deployed routing policy. In
Example~\ref{case:example-9}, both variants answer correctly in one turn, but
\ours{} selects DeepSeek V4-Flash and RIVE w/o response selects DeepSeek
V4-Pro. In this case, \ours{} incurs 95.9\% lower episode cost than RIVE w/o
response.

Example~\ref{case:example-10} provides the clearest accuracy contrast. \ours{}
selects Qwen3.5-9B and returns the correct acceleration, $-120$, in one turn.
RIVE w/o response selects OpenAI GPT-4o-mini, repeats its differentiation plan
through turn nine, and submits $-240$ at turn ten. The ablated variant is
incorrect and costs 15.6$\times$ as much. Together, these cases illustrate how
response-informed training guides model allocation across both accuracy and
cost.

\begin{casebox}
\refstepcounter{caseexample}\label{case:example-7}
\textbf{Example \thecaseexample. Sum of proper divisors}\par
\textbf{Task.} The sum of the proper divisors of 18 is 21. What is the sum of the proper divisors of 198?
\divider\textbf{Turn-by-turn model choices and actions}\par
\begingroup\small\setlength{\tabcolsep}{5pt}\renewcommand{\arraystretch}{1.12}\arrayrulecolor{caseborder}
\begin{tabular}{@{}L{.055\linewidth}L{.44\linewidth}L{.44\linewidth}@{}}
\textbf{Turn} & \textbf{RIVET} & \textbf{RIVE w/o response}\\\hline
1 & \step{DeepSeek V4-Flash}{Runs a divisor-sum function; tool output: 270.} & \step{DeepSeek V4-Flash}{Runs the identical function; tool output: 270.}\\[3pt]
2 & \step{Qwen3.5-9B}{Reruns the function; tool output: 270.} & \step{Qwen3.5-9B}{Reruns the identical function; tool output: 270.}\\[3pt]
3 & \step{DeepSeek V4-Flash}{Switches to DeepSeek V4-Flash and submits 270.} & \step{Qwen3.5-9B}{Keeps Qwen3.5-9B and recomputes; output: 270.}\\[3pt]
4 & \textcolor{muted}{Episode ended} & \step{Qwen3.5-9B}{Repeats the computation; output: 270.}\\[3pt]
5 & \textcolor{muted}{Episode ended} & \step{Qwen3.5-9B}{Changes the function and subtracts 198; output: 72.}\\[3pt]
6 & \textcolor{muted}{Episode ended} & \step{Qwen3.5-9B}{Restores the earlier function; output: 270.}\\[3pt]
7 & \textcolor{muted}{Episode ended} & \step{Qwen3.5-9B}{Explicitly sums the proper divisors; output: 270.}\\[3pt]
8 & \textcolor{muted}{Episode ended} & \step{Qwen3.5-9B}{Includes verification code and submits 270.}\\[3pt]
\end{tabular}\endgroup
\divider\textbf{Selected responses and tool outputs}\par
\textbf{T1, RIVET / DeepSeek V4-Flash (response excerpt):}\par
{\small\raggedright for i in range(1, n):\newline if n \% i == 0:\newline total += i\par}
\textbf{T3, RIVET / DeepSeek V4-Flash (response excerpt):}\par
{\small\raggedright \textless{}solution\textgreater{}270\textless{}/solution\textgreater{}\par}
\textbf{T5, RIVE w/o response / Qwen3.5-9B (response excerpt):}\par
{\small\raggedright return total - n  \# subtract n itself to get proper divisors only\par}
\textbf{T5, RIVE w/o response / Qwen3.5-9B (tool output excerpt):}\par
{\small\raggedright 72\newline None\par}
\divider\textbf{Final answers}\par
\textbf{RIVET:} 270. \textbf{Correct.} 3 turns; Cost: \$0.000230.\par
\textbf{RIVE w/o response:} 270. \textbf{Correct.} 8 turns; Cost: \$0.001073.\par
\end{casebox}
{\footnotesize Task math:84.\par}\vspace{10pt}
\begin{casebox}
\refstepcounter{caseexample}\label{case:example-8}
\textbf{Example \thecaseexample. Evaluating a defined operation}\par
\textbf{Task.} Let $a\star b=a^b-ab$. If $2\star x=22$, find $x$. Options: A: 5; B: 22; C: 11; D: 6.
\divider\textbf{Turn-by-turn model choices and actions}\par
\begingroup\small\setlength{\tabcolsep}{5pt}\renewcommand{\arraystretch}{1.12}\arrayrulecolor{caseborder}
\begin{tabular}{@{}L{.055\linewidth}L{.44\linewidth}L{.44\linewidth}@{}}
\textbf{Turn} & \textbf{RIVET} & \textbf{RIVE w/o response}\\\hline
1 & \step{OpenAI GPT-4o-mini}{Rewrites the equation and plans to test the options.} & \step{OpenAI GPT-4o-mini}{Produces the identical equation and plan.}\\[3pt]
2 & \step{Qwen3.5-9B}{Switches to Qwen3.5-9B and evaluates each option.} & \step{OpenAI GPT-4o-mini}{Keeps OpenAI GPT-4o-mini and evaluates each option.}\\[3pt]
3 & \step{DeepSeek V4-Flash}{Selects DeepSeek V4-Flash and submits 5.} & \step{OpenAI GPT-4o-mini}{Keeps OpenAI GPT-4o-mini and submits 5.}\\[3pt]
\end{tabular}\endgroup
\divider\textbf{Selected responses and tool outputs}\par
\textbf{T1, RIVET / OpenAI GPT-4o-mini (response excerpt):}\par
{\small\raggedright Next, I will evaluate this equation for the given options: A) 5, B) 22, C) 11, D) 6.\par}
\textbf{T2, RIVET / Qwen3.5-9B (tool output excerpt):}\par
{\small\raggedright x = 5: 2\textasciicircum{}x - 2x = 22\par}
\textbf{T2, RIVE w/o response / OpenAI GPT-4o-mini (tool output excerpt):}\par
{\small\raggedright \{5: 22, 22: 4194260, 11: 2026, 6: 52\}\par}
\textbf{T3, RIVET / DeepSeek V4-Flash (response excerpt):}\par
{\small\raggedright \textless{}solution\textgreater{}5\textless{}/solution\textgreater{}\par}
\divider\textbf{Final answers}\par
\textbf{RIVET:} 5. \textbf{Correct.} 3 turns; Cost: \$0.000343.\par
\textbf{RIVE w/o response:} 5. \textbf{Correct.} 3 turns; Cost: \$0.000453.\par
\end{casebox}
{\footnotesize Task mmlu:7.\par}\vspace{10pt}
\begin{casebox}
\refstepcounter{caseexample}\label{case:example-9}
\textbf{Example \thecaseexample. Minimizing a sum with fixed GCD and LCM}\par
\textbf{Task.} Positive integers $m,n$ have greatest common divisor 6 and least common multiple 126. What is the least possible value of $m+n$?
\divider\textbf{Turn-by-turn model choices and actions}\par
\begingroup\small\setlength{\tabcolsep}{5pt}\renewcommand{\arraystretch}{1.12}\arrayrulecolor{caseborder}
\begin{tabular}{@{}L{.055\linewidth}L{.44\linewidth}L{.44\linewidth}@{}}
\textbf{Turn} & \textbf{RIVET} & \textbf{RIVE w/o response}\\\hline
1 & \step{DeepSeek V4-Flash}{Uses DeepSeek V4-Flash and submits 60.} & \step{DeepSeek V4-Pro}{Uses DeepSeek V4-Pro, derives the factor pairs, and submits 60.}\\[3pt]
\end{tabular}\endgroup
\divider\textbf{Selected responses and tool outputs}\par
\textbf{T1, RIVET / DeepSeek V4-Flash (response excerpt):}\par
{\small\raggedright \textless{}solution\textgreater{}60\textless{}/solution\textgreater{}\par}
\textbf{T1, RIVE w/o response / DeepSeek V4-Pro (response excerpt):}\par
{\small\raggedright Coprime factor pairs of 21: (1, 21) and (3, 7).\par}
\textbf{T1, RIVE w/o response / DeepSeek V4-Pro (response excerpt):}\par
{\small\raggedright Corresponding (m, n): (6, 126) and (18, 42). Sums: 132 and 60.\par}
\divider\textbf{Final answers}\par
\textbf{RIVET:} 60. \textbf{Correct.} 1 turn; Cost: \$0.000087.\par
\textbf{RIVE w/o response:} 60. \textbf{Correct.} 1 turn; Cost: \$0.002121.\par
\end{casebox}
{\footnotesize Task math:74.\par}\vspace{10pt}
\begin{casebox}
\refstepcounter{caseexample}\label{case:example-10}
\textbf{Example \thecaseexample. Helicopter acceleration}\par
\textbf{Task.} A helicopter has height $s(t)=300t-4t^3$ metres, with $t$ in minutes. Find its acceleration at $t=5$, in $\mathrm{m}/\mathrm{min}^2$. Return an integer.
\divider\textbf{Turn-by-turn model choices and actions}\par
\begingroup\small\setlength{\tabcolsep}{5pt}\renewcommand{\arraystretch}{1.12}\arrayrulecolor{caseborder}
\begin{tabular}{@{}L{.055\linewidth}L{.44\linewidth}L{.44\linewidth}@{}}
\textbf{Turn} & \textbf{RIVET} & \textbf{RIVE w/o response}\\\hline
1 & \step{Qwen3.5-9B}{Uses Qwen3.5-9B and submits -120.} & \step{OpenAI GPT-4o-mini}{Plans to differentiate the height function.}\\[3pt]
2 & \textcolor{muted}{Episode ended} & \step{OpenAI GPT-4o-mini}{Repeats the differentiation plan without a tool call.}\\[3pt]
3 & \textcolor{muted}{Episode ended} & \step{OpenAI GPT-4o-mini}{Repeats the differentiation plan without a tool call.}\\[3pt]
4 & \textcolor{muted}{Episode ended} & \step{OpenAI GPT-4o-mini}{Repeats the differentiation plan without a tool call.}\\[3pt]
5 & \textcolor{muted}{Episode ended} & \step{OpenAI GPT-4o-mini}{Repeats the differentiation plan without a tool call.}\\[3pt]
6 & \textcolor{muted}{Episode ended} & \step{OpenAI GPT-4o-mini}{Repeats the differentiation plan without a tool call.}\\[3pt]
7 & \textcolor{muted}{Episode ended} & \step{OpenAI GPT-4o-mini}{Repeats the differentiation plan without a tool call.}\\[3pt]
8 & \textcolor{muted}{Episode ended} & \step{OpenAI GPT-4o-mini}{Repeats the differentiation plan without a tool call.}\\[3pt]
9 & \textcolor{muted}{Episode ended} & \step{OpenAI GPT-4o-mini}{Repeats the differentiation plan without a tool call.}\\[3pt]
10 & \textcolor{muted}{Episode ended} & \step{OpenAI GPT-4o-mini}{Includes symbolic-differentiation code and submits -240.}\\[3pt]
\end{tabular}\endgroup
\divider\textbf{Selected responses and tool outputs}\par
\textbf{T1, RIVET / Qwen3.5-9B (response excerpt):}\par
{\small\raggedright \textless{}solution\textgreater{}-120\textless{}/solution\textgreater{}\par}
\textbf{T9, RIVE w/o response / OpenAI GPT-4o-mini (response excerpt):}\par
{\small\raggedright Let's execute the calculations.\par}
\textbf{T10, RIVE w/o response / OpenAI GPT-4o-mini (response excerpt):}\par
{\small\raggedright a = sp.diff(v, t)\par}
\textbf{T10, RIVE w/o response / OpenAI GPT-4o-mini (response excerpt):}\par
{\small\raggedright \textless{}solution\textgreater{}\newline -240\newline \textless{}/solution\textgreater{}\par}
\divider\textbf{Final answers}\par
\textbf{RIVET:} -120. \textbf{Correct.} 1 turn; Cost: \$0.000084.\par
\textbf{RIVE w/o response:} -240. \textbf{Incorrect.} 10 turns; Cost: \$0.001307.\par
\end{casebox}
{\footnotesize Task theoremqa:41.\par}\vspace{10pt}
\endgroup

\section{Use of Generative AI}
We used generative AI tools to polish the writing, explore related
literature, assist with experiment implementation, and help check
mathematical derivations. All AI-assisted content was reviewed and verified
by the authors, who take full responsibility for the entire manuscript.

\end{document}